\documentclass[twoside]{article}

\usepackage[accepted]{aistats2026}
\usepackage[round]{natbib}

\usepackage[utf8]{inputenc} 
\usepackage[T1,T2A,T5]{fontenc}
\usepackage[main=english,french,spanish,german,italian,polish,turkish,danish,russian,vietnamese,shorthands=off]{babel}
\AtBeginDocument{\selectlanguage{english}} 
\usepackage{hyperref}
\usepackage{url}

\usepackage{amsmath,amsfonts,bm}
\usepackage{mathtools}
\usepackage{subfig}
\usepackage{graphicx}

\usepackage{booktabs} 
\usepackage{multirow}
\usepackage{multicol}
\usepackage{diagbox}

\usepackage{amsthm}
\usepackage{makecell} 
\newtheorem{theorem}{Theorem}

\usepackage{wrapfig}   
\usepackage[font={small}]{caption}   

\newcommand{\ru}[1]{{\fontencoding{T2A}\selectfont #1}}
\newcommand{\pltt}[1]{{\fontencoding{T1}\selectfont\ttfamily #1}}
\newcommand{\vi}[1]{{\fontencoding{T5}\selectfont #1}}

\usepackage{CJKutf8}
\newcommand{\zhtt}[1]{\begin{CJK*}{UTF8}{gbsn}\ttfamily #1\end{CJK*}}
\newcommand{\jatt}[1]{\begin{CJK*}{UTF8}{min}\ttfamily #1\end{CJK*}} 
\newcommand{\kott}[1]{\begin{CJK*}{UTF8}{mj}\ttfamily #1\end{CJK*}}

\usepackage{amssymb}
\usepackage{xcolor}
\newcommand{\deltaup}[1]{{\scriptsize (\textcolor{red}{$\blacktriangle$}\,#1)}}
\newcommand{\deltadown}[1]{{\scriptsize (\textcolor{blue}{$\blacktriangledown$}\,#1)}}
\newcommand{\uptriangle}{\textcolor{red}{$\blacktriangle$}}
\newcommand{\down}{\textcolor{blue}{$\blacktriangledown$}}

\begin{document}

%

%

\twocolumn[

\aistatstitle{Topological Alignment of Shared Vision-Language Embedding Space}

\aistatsauthor{ Junwon You \And Dasol Kang \And  Jae-Hun Jung }

\aistatsaddress{ Department of Mathematics,  \\ POSTECH \And  BootCamp, Google \And Department of Mathematics, \\ POSTECH } ]

\begin{abstract}
Contrastive Vision-Language Models (VLMs) have demonstrated strong zero-shot capabilities. 
However, their cross-modal alignment remains biased toward English due to limited multilingual multimodal data. 
Recent multilingual extensions have alleviated this gap but enforce instance-level alignment while neglecting the global geometry of the shared embedding space. 
We address this problem by introducing \textbf{ToMCLIP} (\textbf{To}pological Alignment for \textbf{M}ultilingual \textbf{CLIP}), a topology-aware framework aligning embedding spaces with topology-preserving constraints. 
The proposed method applies persistent homology to define a topological alignment loss and approximates persistence diagram with theoretical error bounds using graph sparsification strategy. 
This work validates the proposed approach, showing enhanced structural coherence of multilingual representations, higher zero-shot accuracy on the CIFAR-100, and stronger multilingual retrieval performance on the xFlickr\&CO. 
Beyond VLMs, the proposed approach provides a general method for incorporating topological alignment into representation learning. 
Code is available at \url{https://github.com/junwon0/ToMCLIP.git}.
\end{abstract}

\section{INTRODUCTION}

Contrastive Vision-Language Models (VLMs), such as CLIP~\citep{radford2021learning} and ALIGN~\citep{jia2021scaling} have demonstrated strong zero-shot transfer capabilities by learning a shared embedding space for images and texts~\citep{bordes2024introduction}. 
These models align paired samples through contrastive learning, enabling diverse downstream tasks without task-specific supervision. 
Although autoregressive multimodal large language models such as LLaVA~\citep{liu2024visual}, Qwen-VL~\citep{bai2023qwen}, and Gemini~\citep{team2023gemini} have recently achieved vision-language understanding via generative training, contrastive VLMs remain effective for retrieval tasks and computational efficiency.

Despite recent multilingual extensions~\citep{carlsson2022cross, chen2023mclip, yang2024embracing}, representation spaces remain structurally misaligned. 
Most approaches enforce instance-level alignment via distillation or continual learning, but they fail to preserve the global geometry in the shared embedding space. 
This structural misalignment causes unstable cross-lingual retrieval and inconsistent semantic clustering.

\begin{figure*}[!ht]
  \centering
  \includegraphics[width=0.95\linewidth]{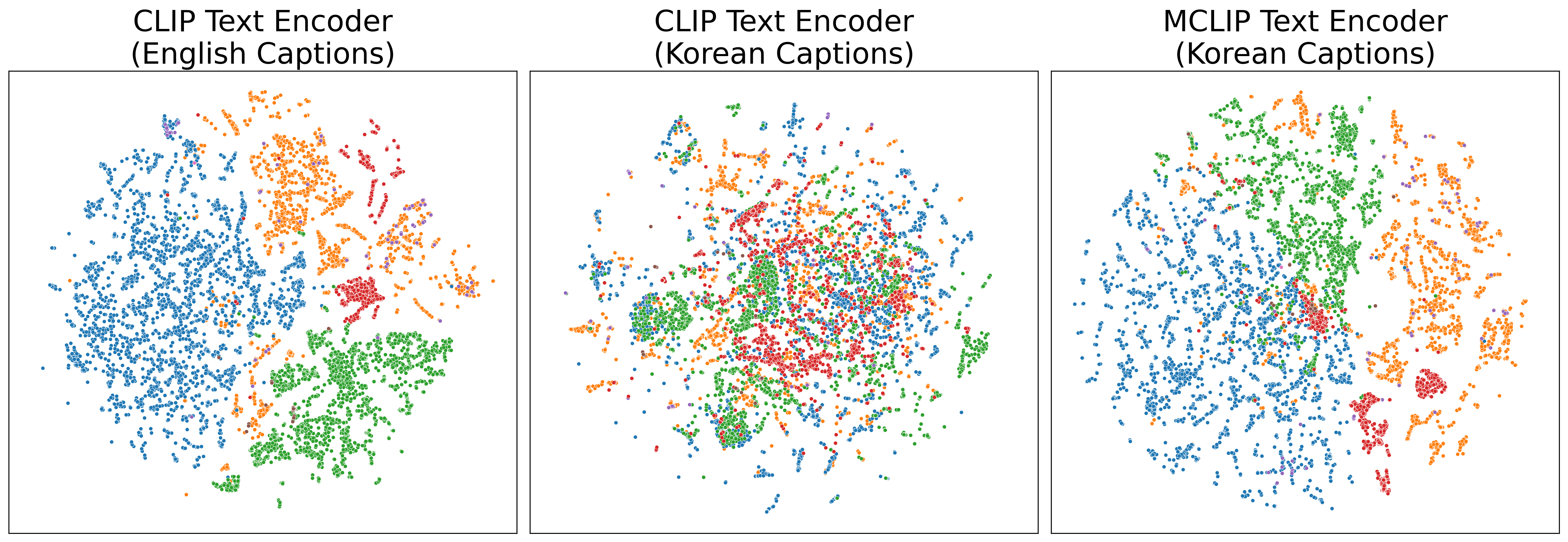} 
  \caption{Visualization of text embeddings (English and Korean) in the latent space using t-SNE~\citep{maaten2008visualizing}, from CLIP and multilingual CLIP (MCLIP; \citealp{carlsson2022cross}) text encoders.  
  The Fashion Product Images dataset~\citep{param_aggarwal_2019} was used, where the \texttt{productDisplayName} field serves as the input caption to the text encoders. Colors indicate the corresponding \texttt{masterCategory} of each product.}
  \label{fig:motivation_embedding}
\end{figure*}

As illustrated in Figure~\ref{fig:motivation_embedding}, 
the English and Korean text embeddings produced by the CLIP encoder are not aligned.
Even the multilingual CLIP (MCLIP; \citealp{carlsson2022cross})
fails to achieve cross-lingual alignment, with multiple semantic categories remaining intermixed in the center of the embedding space. 
To address this limitation, we propose \textbf{ToMCLIP}: \textbf{To}pological Alignment for \textbf{M}ultilingual \textbf{CLIP}, a topology-aware training framework that enforces structural consistency across languages using topological data analysis.
This approach is motivated by the hypothesis that performance gaps between English and other languages stem from differences in the topological structure of their latent representations.

The contributions of this work are as follows:
\begin{itemize}
    \item We introduce a topology-aware training framework for multilingual contrastive VLMs. 
    It formalizes the structural misalignment across languages and addresses it with a topological alignment loss that enforces structural alignment in the shared embedding space.
    \item We develop a scalable approximation for persistence diagrams. 
    The approach constructs sparse graphs using MST-based sparsification and provides theoretical error bounds of approximation.
    \item We validate the proposed method using case studies on multilingual vision-language tasks. 
    The experiments reveal improved cross-lingual structural coherence, higher zero-shot accuracy on the CIFAR-100, and stronger multilingual retrieval performance on the xFlickr\&CO.
\end{itemize}

\subsection{Related Work}

Appendix~\ref{app:relatedworks} reviews related work on contrastive VLMs and autoregressive multimodal large language models.


\paragraph{Multilingual Extensions of Contrastive VLMs.}
Various multilingual extensions of contrastive VLMs have been developed, using knowledge distillation, continual learning, or multilingual pretraining to align images and texts across languages.
For example, MCLIP~\citep{carlsson2022cross} trains a single multilingual text encoder using text-only, machine-translation-based distillation to match the original CLIP English text space.
In contrast, mCLIP~\citep{chen2023mclip} retains the dual-encoder design of CLIP but aligns a multilingual text encoder to CLIP via Triangle Cross-modal Knowledge Distillation (TriKD). 
The multilingual text encoder is initialized using contrastive pretraining.
AltCLIP~\citep{chen2023altclip} replaces the original CLIP text encoder with a pretrained multilingual text encoder.
It aligns multilingual text representations to the CLIP image–text space through knowledge distillation and contrastive learning.
Continual language learning approaches~\citep{yang2024embracing} add languages incrementally to mitigate catastrophic forgetting.

\paragraph{Topological Analysis of the Embedding Space.}
Recent studies have emphasized the importance of preserving the topological structure in representation learning~\citep{moor2020topological, trofimov2023learning, zilberstein2024topology}. 
Complementary efforts have employed topological representations enriching representation learning~\citep{carriere2020perslay, papillon2023architectures, wen2024tensor}. 
Building on these advances, topology-aware techniques have been applied in the context of VLMs to improve embedding robustness and generalization~\citep{zhang2024homology, rahim2024topological, huang2025topology}. 
Furthermore, topological representations have proven effective for knowledge distillation and continual learning, where the latent space geometry acts as transferable knowledge~\citep{kim2024topological, wang2024persistence, hai2025topology}.

Despite these advances, topological consistency across multilingual embeddings remains underexplored. 
This work proposes a topological alignment framework that enforces structural coherence between the latent spaces of CLIP and MCLIP using persistent homology. 
We focus on MCLIP as a representative approach, as it relies on limited data and a simple MSE-based distillation objective, in contrast to other methods that require large-scale datasets or computationally expensive contrastive learning.
\begin{figure*}[!ht]
  \centering
  \includegraphics[width=0.9\linewidth]{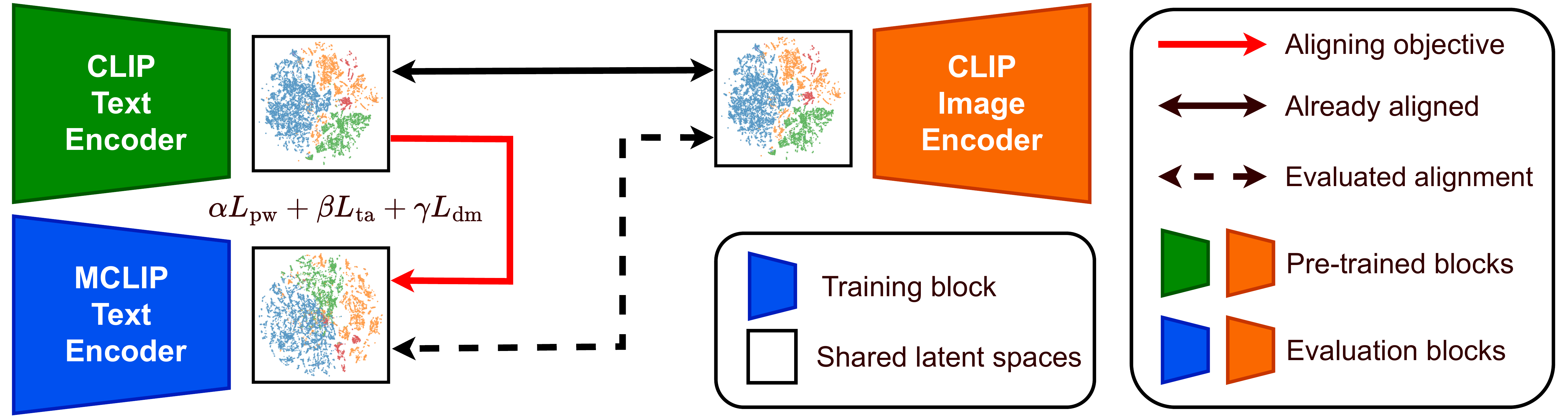} 
  \caption{Overview of the proposed alignment framework between CLIP ($E_T$) and multilingual CLIP (MCLIP; $E_S$) text encoders. 
  $E_S$ is trained to align with the frozen $E_T$ using a combination of loss functions; $L_{\text{pw}}$ enforces point-wise alignment; $L_{\text{ta}}$ and $L_{\text{dm}}$ promote geometric alignment by preserving topological structures. 
  The evaluation is conducted by pairing $E_S$ with the pretrained CLIP image encoder, enabling cross-lingual retrieval in the shared embedding space.}
  \label{fig:tomclip}
\end{figure*}

\section{TOPOLOGICAL ALIGNMENT}
\label{sec:topological_alignment}

Figure~\ref{fig:tomclip} presents an overview of our proposed alignment framework.
Appendix~\ref{app:ph} presents the preliminaries of the persistent homology, including persistence diagrams and the (sliced) Wasserstein distance.

We integrated topological alignment loss with MCLIP~\citep{carlsson2022cross}.
The MCLIP proposes a teacher-student framework that applies machine-translated captions for training. 
A set of English captions $X$ is translated into a target language to form $X^*$. 
The CLIP text encoder $E_T$ (teacher) encodes the original captions $X$, and the MCLIP text encoder $E_S$ (student) encodes the translated captions $X^*$. 
Then $E_S$ is trained to align with the teacher by minimizing the mean squared error (MSE) between the output embeddings:
\begin{equation}
    L_{\mathrm{pw}} = \mathrm{MSE}(E_T(X), E_S(X^*)).
\end{equation}

This approach focuses on point-wise alignment, overlooking the structural consistency of the embedding space across languages. 
Table~\ref{tab:notation_2_1} in Appendix~\ref{app:notations} summarizes the notations used in Section~\ref{sec:topological_alignment}.

\subsection{Topological Alignment Loss}
We introduce a novel topological alignment loss $L_\mathrm{ta}$ that enforces the global structural alignment. 
For a batch of data comprising $N$ image-text pairs $\{(I_i,T_i)\}^N_{i=1}$, the text representations $\{E_T(T_i)\}^N_{i=1}$ form a geometric structure in the embedding space (Figure~\ref{fig:motivation_embedding}).
The MCLIP loss $L_{\mathrm{pw}}$ considers each representation $E_T(T_i)$ independently, ignoring the geometric relationships between the samples.

To address this problem, we compute the persistence diagram $D_T$ from the point cloud $\{E_T(T_i)\}_{i=1}^N$, which summarizes the topological features of the embedding distribution (e.g., connected components and cycles).
Persistence diagrams provide a summary of the global structure of a point cloud.
Given an embedding set, we gradually increase a neighborhood radius and track how connected components merge and how loops appear or disappear.
Each topological feature is recorded as a point $(b, d)$ in the diagram, where $b$ is the radius at which the structure appears (birth) and $d$ is the radius at which it disappears (death).
Features that persist over a long radius range (large $d - b$) correspond to prominent global structures, while short-lived features represent noise.
Similarly, we compute $D_S$ from the point cloud $\{E_S(T_i^*)\}_{i=1}^N$, where $T^*_i$ denotes the translated caption of $T_i$, capturing the structure of the MCLIP.
To align these spaces, we define the topology alignment loss:
\begin{equation}
    L_\mathrm{ta} = \mathrm{SW}_p^{(K)}(D_T, D_S),
\end{equation}
where $\mathrm{SW}_{p}$ denotes the sliced $p$-Wasserstein distance (SWD,~\citealp{bonneel2015sliced}) and $K$ represents the number of projection directions.
The SWD provides a fast, differentiable, and GPU-friendly approximation of the Wasserstein distance, making it suitable as a training loss. 
Minimizing the discrepancy between $D_T$ and $D_S$ enforces both embedding spaces to preserve comparable global topological structures, complementing pointwise matching with structural alignment.

Furthermore, we define a distance matrix loss $L_\mathrm{dm}$ to promote local geometric alignment between the latent spaces.
Given a point cloud $X = \{x_i\}^N_{i=1}$, the pairwise distance matrix is computed as follows: \( (\mathrm{M_X})_{i,j} = \left\Vert x_i-x_j \right\Vert \) for $1 \leq i,j \leq N$ where $\left\Vert \cdot \right\Vert$ denotes the Euclidean ($L_2$) norm. 
The distance matrix loss is defined as follows:
\begin{equation}
    L_\mathrm{dm} = \mathrm{MSE}(M_T,M_S)
\end{equation}
where $M_T$ and $M_S$ denote the distance matrices computed from the point clouds $T = \{E_T(T_i)\}^N_{i=1}$ and $S = \{E_S(T_i^*)\}^N_{i=1}$, respectively.
The total training objective is defined as the weighted sum of three loss components:
\begin{equation}
\label{eq:total_loss}
    L_{\mathrm{total}} = \alpha L_{\mathrm{pw}} + \beta L_{\mathrm{ta}} + \gamma L_{\mathrm{dm}}  ,
\end{equation}
where \( \alpha \), \( \beta \), and \( \gamma \) are hyperparameters that control the relative contributions of each loss term.
Geometry describes local numerical relationships between embeddings, such as pairwise distances, whereas topology captures global structural patterns, including connected components, clusters, and higher-order features.
Accordingly, the pairwise loss $L_{\mathrm{pw}}$ preserves instance-level alignment, the topological alignment loss $L_{\mathrm{ta}}$ enforces global structural consistency across latent spaces, and the distance matrix loss $L_{\mathrm{dm}}$ promotes local geometric alignment by matching pairwise distances.

\paragraph{Stability-Based Justification of the Loss Design.}
Let $X,Y\subset\mathbb{R}^n$ be finite point clouds.
The $k$-dimensional persistence diagrams are denoted by $D^{(k)}_X$ and $D^{(k)}_Y$, respectively.
By the stability theorem, for any $p\ge1$, $C_k\ge1$ exists such that
\begin{equation}
\label{eq:stability}
    W_p\big(D^{(k)}_X, D^{(k)}_Y\big)\ \le\ C_k\, W_p^c(X,Y),
\end{equation}
where $W_p$ is the $p$-Wasserstein distance between diagrams and $W_p^c$ is the $p$-Wasserstein distance between point clouds~\citep{skraba2020wasserstein}.
Thus, if $W_p\big(D^{(k)}_X, D^{(k)}_Y\big)\ge\tau$, then $W_p^c(X,Y)\ge\tau/C_k$.
Therefore, minimizing the distance between persistence diagrams ($\mathcal{L}_{\mathrm{ta}}$) reduces the certified lower bound on the point cloud discrepancy. 
Moreover, because $D^{(0)}$ summarizes the connectivity in the embedding space, minimizing $\mathcal{L}_{\mathrm{ta}}$ between $D^{(0)}$s reduces the cross-lingual semantic cluster misalignment, encouraging semantically equivalent texts to belong to the same cluster.

However, $\mathcal{L}_{\mathrm{ta}}$ and $\mathcal{L}_{\mathrm{dm}}$ are invariant to Euclidean isometries. 
If $Y=RX+t$ with $R\in O(n)$ (i.e., $R^\top R=I$ and $\det R\in\{\pm1\}$) and $t\in\mathbb{R}^n$, then $\mathcal{L}_{\mathrm{ta}}=\mathcal{L}_{\mathrm{dm}}=0$ and $W_p^c(X,Y)$ can be arbitrarily large. 
Hence, these terms alone do not reduce $W_p^c$ or prevent rigid-motion drift. 
Therefore, $\mathcal{L}_{\mathrm{pw}}$ is needed to fix the coordinate frame, while $\mathcal{L}_{\mathrm{ta}}$ aligns the global topology and $\mathcal{L}_{\mathrm{dm}}$ matches the pairwise geometry.

\subsection{Approximating Persistence Diagrams}
\label{sec:approximate_PD}

This work employs two strategies to approximate the persistence diagram of the Vietoris-Rips (Rips) complex with reduced computational overhead:
\begin{itemize}
    \item We restrict the computation to 0-dimensional ($H_0$) features and the birth times of 1-dimensional ($H_1$) features, which can be extracted from the minimal spanning tree (MST)~\citep{kruskal1956shortest} with a union-find~\citep{tarjan1979class}. 
    This eliminates the need to construct the full Rips complex. 
    Prior work has confirmed that $H_0$ features are sufficient to capture the topological structure of latent representations~\citep{moor2020topological,kim2024topological}.
    \item To reduce the computational cost of MST further, we build a sparse graph from pairwise distances between embeddings, limiting the number of candidate edges.
\end{itemize}

This approximation reduces memory and time, enabling persistence diagrams in large-scale training. 
For a point cloud with $N$ points, computing the Rips complex has an exponential complexity of up to $\mathcal{O}(N^{k+1})$ for $k$-dimensional simplices.
Persistent homology via boundary-matrix reduction has a worst case time of $\mathcal{O}(m^3)$ and a memory of $\mathcal{O}(m^2)$~\citep{otter2017roadmap}, where $m$ denotes the total number of simplices in the filtration. 
Consequently, computing $H_0$ has a cost of $m=\mathcal{O}(N^2)$ up to $\mathcal{O}(N^6)$, whereas computing $H_1$ costs $m=\mathcal{O}(N^3)$ up to $\mathcal{O}(N^9)$. 
However, $H_0$ and the birth time of $H_1$ features can be computed via the MST, which has a computational complexity of $\mathcal{O}(E\log V)$, where $V$ denotes the number of vertices and $E$ represents the number of edges~\citep{cormen2022introduction}.
Notably, for $H_0$, only $N-1$ edges are necessary to determine the death time, corresponding to the edges of the MST, out of a total $\binom{N}{2}$ edges in the fully connected graph.
Therefore, constructing the MST over a complete graph is computationally inefficient.
To mitigate this problem, we construct a sparse graph $G_\epsilon=(V,E_\epsilon)$ from a point cloud $X = \{x_1, \cdots, x_N\} \subset (\mathbb{R}^n, d)$, where $V = \{x_i\}_{i=1}^N$ and $E_\epsilon = \{(x_i,x_j) \mid d(x_i,x_j) \leq \epsilon \}$, with $d$ denoting a metric (e.g., Euclidean distance).
This sparsification reduces the number of edges while retaining a sufficient topological structure to approximate the persistence diagram.

We calculate the upper bound on the approximation error of the proposed method.
We construct a weighted complete graph $G = (V, E, \omega)$ from a point cloud $X$, where $V = X$, $E = \{(x_i, x_j) \mid x_i, x_j \in X,\, i \neq j \}$, and the weight function $\omega : E \to \mathbb{R}_{\ge0}$ is defined as  
\begin{equation}
    \omega((x_i, x_j)) = \frac{d(x_i, x_j)}{M},
\end{equation}
where $M = \max\limits_{(x_i,x_j) \in E} d(x_i, x_j)$.  
By construction, $0 \leq \omega(e) \leq 1$ for all $e \in E$.  

\begin{theorem}
    Let $0 \le \epsilon \le 1$ and $G_\epsilon = (V,E,\omega_\epsilon)$,
    \begin{equation}
        \omega_\epsilon(e) =
        \begin{cases}
            \omega(e), & \text{if } \omega(e) \le \epsilon,\\[4pt]
            1,         & \text{if } \omega(e) > \epsilon.
        \end{cases}
    \end{equation}
    Let
        $m(\epsilon) \coloneqq 
        \#\bigl\{ (0,d) \in D_0^{\mathrm{Rips}}(G) \mid \epsilon < d < \infty \bigr\}$,
    i.e., the number of finite $0$-dimensional persistence points of $G$ 
    whose death times exceed $\epsilon$.  
    Then,
    \begin{equation}
        W_p\!\bigl(D_0^{\mathrm{Rips}}(G),\,D_0^{\mathrm{Rips}}(G_\epsilon)\bigr)
        \;\le\; m(\epsilon)^{1/p}\,(1-\epsilon)
    \end{equation}
    and $0 \leq m(\epsilon) \leq N-1$ where $W_p$ denotes the $p$-Wasserstein distance.
\end{theorem}

Appendix~\ref{app:proof} presents the proof of this theorem.
Let $c(\epsilon)$ denote the number of connected components in $\mathrm{VR}_\epsilon(G)$ which is equal to $m(\epsilon)+1$. 
Therefore,
\begin{equation}
   W_p\!\bigl(D_0^{\mathrm{Rips}}(G),\,D_0^{\mathrm{Rips}}(G_\epsilon)\bigr)
   \;\le\; (c(\epsilon)-1)^{1/p}\,(1-\epsilon).
\end{equation}

\begin{table*}[!ht]
\centering
\caption{
Average connected components $c(\epsilon)$ and sparsity by $\lambda$ value on random point clouds ($n=512$, 10 trials).
Average number of connected components $c(\epsilon)$ and sparsity for different $\lambda$ values on random point clouds ($N = 512$, 10 trials).
According to the theoretical upper bound of our approximation, the approximation error becomes exactly zero when the sparsified graph $G_{\epsilon}$ forms a single connected component ($c(\epsilon) = 1$).
The sparsity term indicates the percentage of edges retained when constructing the sparsified graph used to approximate the minimum spanning tree, whose computational cost is $\mathcal{O}(E \log V)$.
}
\label{tab:connectivity_sparsity_matrix}
\small
\setlength{\tabcolsep}{4pt}
\resizebox{0.95\textwidth}{!}{%
\begin{tabular}{c *{5}{c} *{5}{c} @{\hspace{8pt}} *{5}{c} *{5}{c}}
\toprule
\multirow{4}{*}{$N$}& \multicolumn{10}{c}{Connected components $c(\epsilon)$}
& \multicolumn{10}{c}{Sparsity} \\
\cmidrule(lr){2-11}\cmidrule(lr){12-21}
& \multicolumn{5}{c}{Uniform ($\lambda$)} & \multicolumn{5}{c}{Gaussian ($\lambda$)}
& \multicolumn{5}{c}{Uniform ($\lambda$)} & \multicolumn{5}{c}{Gaussian ($\lambda$)} \\
\cmidrule(lr){2-6}\cmidrule(lr){7-11}\cmidrule(lr){12-16}\cmidrule(lr){17-21}
    & $1.0$ & $0.5$ & $0.0$ & $-0.5$ & $-1.0$
    & $1.0$ & $0.5$ & $0.0$ & $-0.5$ & $-1.0$
    & $1.0$ & $0.5$ & $0.0$ & $-0.5$ & $-1.0$
    & $1.0$ & $0.5$ & $0.0$ & $-0.5$ & $-1.0$ \\
\midrule
64  & 1.6 & 1.1 & 1.0 & 1.0 & 1.0 & 4.1 & 1.4 & 1.1 & 1.0 & 1.0
    & 0.158 & 0.306 & 0.496 & 0.690 & 0.840 & 0.157 & 0.309 & 0.504 & 0.693 & 0.840  \\
128 & 1.7 & 1.0 & 1.0 & 1.0 & 1.0 & 3.1 & 1.2 & 1.0 & 1.0 & 1.0
    & 0.160 & 0.310 & 0.499 & 0.692 & 0.841 & 0.160 & 0.311 & 0.502 & 0.694 & 0.841 \\
256 & 1.1 & 1.0 & 1.0 & 1.0 & 1.0 & 3.2 & 1.2 & 1.1 & 1.0 & 1.0
    & 0.159 & 0.308 & 0.499 & 0.692 & 0.841 & 0.159 & 0.310 & 0.503 & 0.693 & 0.842  \\
512 & 1.0 & 1.0 & 1.0 & 1.0 & 1.0 & 2.2 & 1.0 & 1.0 & 1.0 & 1.0
    & 0.158 & 0.308 & 0.499 & 0.690 & 0.841 & 0.159 & 0.310 & 0.502 & 0.692 & 0.841 \\
\bottomrule
\end{tabular}}
\end{table*}
\begin{table*}[!ht]
\centering
\caption{Top-10 accuracy (\%) of zero-shot classification on CIFAR-100 across 13 languages (Full vs. Low).}
\label{tab:cifar100_lang_top10}
\resizebox{\textwidth}{!}{%
\begin{tabular}{llcccccccccccccc}
\toprule
\multirow{3}{*}{Setting} & \multirow{3}{*}{Model} & \multicolumn{13}{c}{Languages (13)} & \multirow{3}{*}{\textit{Avg}} \\
\cmidrule(lr){3-15}
& & En & Fr & Es & De & It & Ru & Pl & Tr & Da & Ja & Zh & Ko & Vi \\
\midrule
 & CLIP & 91.06 & 66.18 & 63.69 & 64.05 & 49.33 & 11.95 & 22.03 & 24.73 & 32.42 & 32.80 & 21.56 & 12.38 & 15.32 & 39.04 \\
& MCLIP & 91.97 & 85.66 & 87.10 & 85.74 & 88.23 & 87.98 & \textbf{85.38} & 87.65 & 87.83 & 53.60 & 89.50 & 87.20 & 86.26 & 84.93 \\
\textbf{Full data} & ToMCLIP($L_{\text{dm}}$) & \textbf{91.99} & 84.77 & 84.63 & \textbf{89.63} & 86.17 & 87.78 & 84.86 & 87.35 & 86.88 & 56.27 & 88.11 & 87.94 & 86.98 & 84.87 \\
\textbf{(2M smaples)} & ToMCLIP($L_{\text{ta}}$) & 91.48 & 85.41 & 84.23 & 87.85 & 88.49 & \textbf{89.43} & 84.35 & \textbf{88.76} & \textbf{87.98} & \textbf{58.57} & \textbf{89.75} & \textbf{88.76} & \textbf{89.41} & 85.73 \\
 & ToMCLIP & 91.40 & \textbf{87.59} & \textbf{87.37} & 89.30 & \textbf{89.11} & 87.66 & 83.59 & 88.59 & 87.79 & 57.95 & 88.68 & 88.36 & 88.17 & \textbf{85.81} \\
 \midrule
 & CLIP & \textbf{91.06} & 66.18 & 63.69 & 64.05 & 49.33 & 11.95 & 22.03 & 24.73 & 32.42 & 32.80 & 21.56 & 12.38 & 15.32 & 39.04 \\
& MCLIP & 79.72 & 67.60 & 62.20 & 71.41 & 59.68 & 69.80 & 64.55 & 58.71 & 73.31 & 60.68 & 78.27 & 65.43 & 71.38 & 67.90 \\
\textbf{Low resource} & ToMCLIP($L_{\text{dm}}$) & 79.46 & 67.99 & 62.51 & 70.81 & 60.75 & 69.30 & 64.02 & 57.21 & 72.64 & 59.20 & 77.43 & 67.42 & 70.07 & 67.60 \\
\textbf{(1\% subset)} & ToMCLIP($L_{\text{ta}}$) & 80.00 & 67.37 & 62.66 & 70.09 & 60.88 & 70.31 & 65.22 & 59.50 & 72.68 & 60.94 & 77.36 & 67.01 & \textbf{73.37} & 68.26 \\
& ToMCLIP & 80.75 & \textbf{68.56} & \textbf{63.85} & \textbf{71.49} & \textbf{62.91} & \textbf{71.23} & \textbf{65.50} & \textbf{60.80} & \textbf{73.75} & \textbf{62.39} & \textbf{78.82} & \textbf{67.96} & 72.44 & \textbf{69.26} \\
\bottomrule
\end{tabular}
}
\end{table*}
 
As $\epsilon$ increases, more edges are retained, sparsity decreases, and the number of connected components $c(\epsilon)$ monotonically decreases. 
In particular, a critical value $\epsilon_*$ exists such that $c(\epsilon)=1$ for all $\epsilon \ge \epsilon_*$, (i.e., $\mathrm{VR}_\epsilon(G)$ becomes connected). 
From an algorithmic perspective, the critical trade-off lies in selecting $\epsilon$ so that $\mathrm{VR}_\epsilon(G)$ remains sparse while maintaining a small number of connected components. 
The experiments confirm that moderate values of $\epsilon$ already achieve near connectivity with a low edge density, making the sparsification highly effective in practice (Section~\ref{sec:conn_spars}).

\section{RESULTS}

\subsection{Connectivity and Sparsity Analysis of Approximation Method}
\label{sec:conn_spars}

We evaluated the effect of the threshold parameter $\epsilon$ on the sparsity and connectivity 
of the sparsified graph $G_\epsilon = (V,E_\epsilon)$. 
Across uniform and Gaussian random point clouds in $\mathbb{R}^n$ 
with $N \in \{64,128,256,512\}$, we measured the number of connected components $c(\epsilon)$ 
and the average sparsity ($|E_\epsilon| / \binom{N}{2}$) when $\epsilon = \mu -\lambda \sigma$ 
for $\lambda \in \{ 1.0, 0.5, 0, -0.5, -1.0 \}$, where $\mu$ and $\sigma$ denote the mean and standard deviation of all weights in $G$.

Table~\ref{tab:connectivity_sparsity_matrix} reveals a clear monotonic trade-off.
As $\lambda$ decreases, the threshold $\epsilon=\mu - \lambda\sigma$ increases, 
leading to a higher edge density and fewer connected components. 
At a large positive value of $\lambda$, graphs are sparse but fragmented into multiple components, 
particularly for Gaussian point clouds, which exhibit a stronger central concentration. 
As $\lambda$ becomes smaller, the graphs quickly become connected ($c(\epsilon)=1$), and sparsity rises above 0.5. 
At $\lambda=0.5$, the graphs achieve near connectivity across all $N$, 
while retaining only about 30\% of the edges. 
ToMCLIP adopts this setting, as it offers an effective balance between sparsity and connectivity.

\begin{table}[!ht]
\centering
\caption{Average Top-$k$ accuracy (\%) of the zero-shot classification on CIFAR-100 across 13 languages.}
\label{tab:cifar100_multilingual_avg}
\small
\resizebox{0.48\textwidth}{!}{%
\begin{tabular}{lcccccc}
\toprule
{} & \multicolumn{3}{c}{\textbf{Low resource}} & \multicolumn{3}{c}{\textbf{Full data}} \\
\cmidrule(lr){2-4} \cmidrule(lr){5-7}
 & Top-1 & Top-5 & Top-10 & Top-1 & Top-5 & Top-10 \\
\midrule
CLIP & 20.29 & 32.47 & 39.04 & 20.29 & 32.47 & 39.04 \\
MCLIP & 30.21 & 56.67 & 67.90 & 50.72 & 76.49 & 84.93 \\
ToMCLIP($L_{\text{dm}}$) & 31.12 & 56.47 & 67.60 & 50.53 & 75.84 & 84.87 \\
ToMCLIP($L_{\text{ta}}$) & 30.45 & 57.14 & 68.26 & 50.73 & 77.12 & 85.73 \\
ToMCLIP & \textbf{31.91} & \textbf{58.15} & \textbf{69.26} & \textbf{51.32} & \textbf{77.46} & \textbf{85.81} \\
\bottomrule
\end{tabular}}
\end{table}

\subsection{Evaluation of ToMCLIP}
We evaluate MCLIP~\citep{carlsson2022cross} and ToMCLIP under two training conditions: (1) using the full available dataset and (2) using only 1\% of the data for the low-resource setting. 
This setup is designed to mimic realistic situations where only a few of multilingual annotated data are available for training.
The experiments use the multilingual caption corpus from~\citet{carlsson2022cross} (2M precomputed embeddings), augmented with Korean translations. 
We employ XLM-RoBERTa~\citep{conneau2019unsupervised} for the multilingual text encoder and ViT-B/32 variant~\citep{radford2021learning} for the CLIP model. 
Appendix~\ref{app:data} provides details on dataset preparation, training and evaluation.
The ToMCLIP($L_{\text{dm}}$), ToMCLIP($L_{\text{ta}}$), and ToMCLIP denote models trained with the proposed total loss \(L_{\text{total}}\) (\ref{eq:total_loss}) using the coefficients \((\alpha,\beta,\gamma)=(1,0.01,0)\), \((1,0,0.01)\), and \((1,0.01,0.01)\), respectively.
In computing $L_{\text{ta}}$, this approach is restricted to 0-dimensional features ($H_0$) in the persistence diagram. 
The birth times of 1-dimensional features, corresponding to the remaining edge weights after 0-dimensional features merge, provide no substantial additional information, because the pairwise distance matrix is already optimized using the MSE loss.
The results are from a single training run, consistent with standard research practices~\citep{radford2021learning, carlsson2022cross, chen2023mclip, yang2024embracing}.
For the 1\% low-resource setting, this work reports the mean over three independent runs.
Appendix~\ref{app:vit+} presents the results replacing the CLIP image encoder with ViT-B/16+~\citep{cherti2023reproducible}.

\begin{figure*}[!ht]
  \centering
  \includegraphics[width=0.88\textwidth]{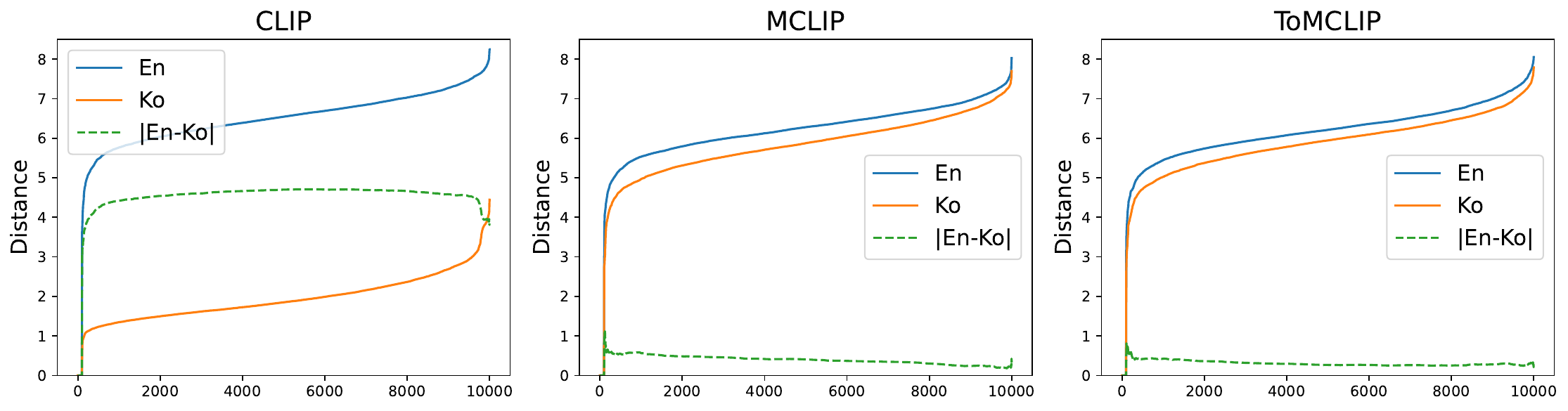}
  \caption{Sorted pairwise distance curves of English (En) vs.\ Korean (Ko) embeddings. 
  }
  \label{fig:pairwise_en_ko}
\end{figure*}

\begin{figure*}[!th]
  \centering
  \includegraphics[width=1\textwidth]{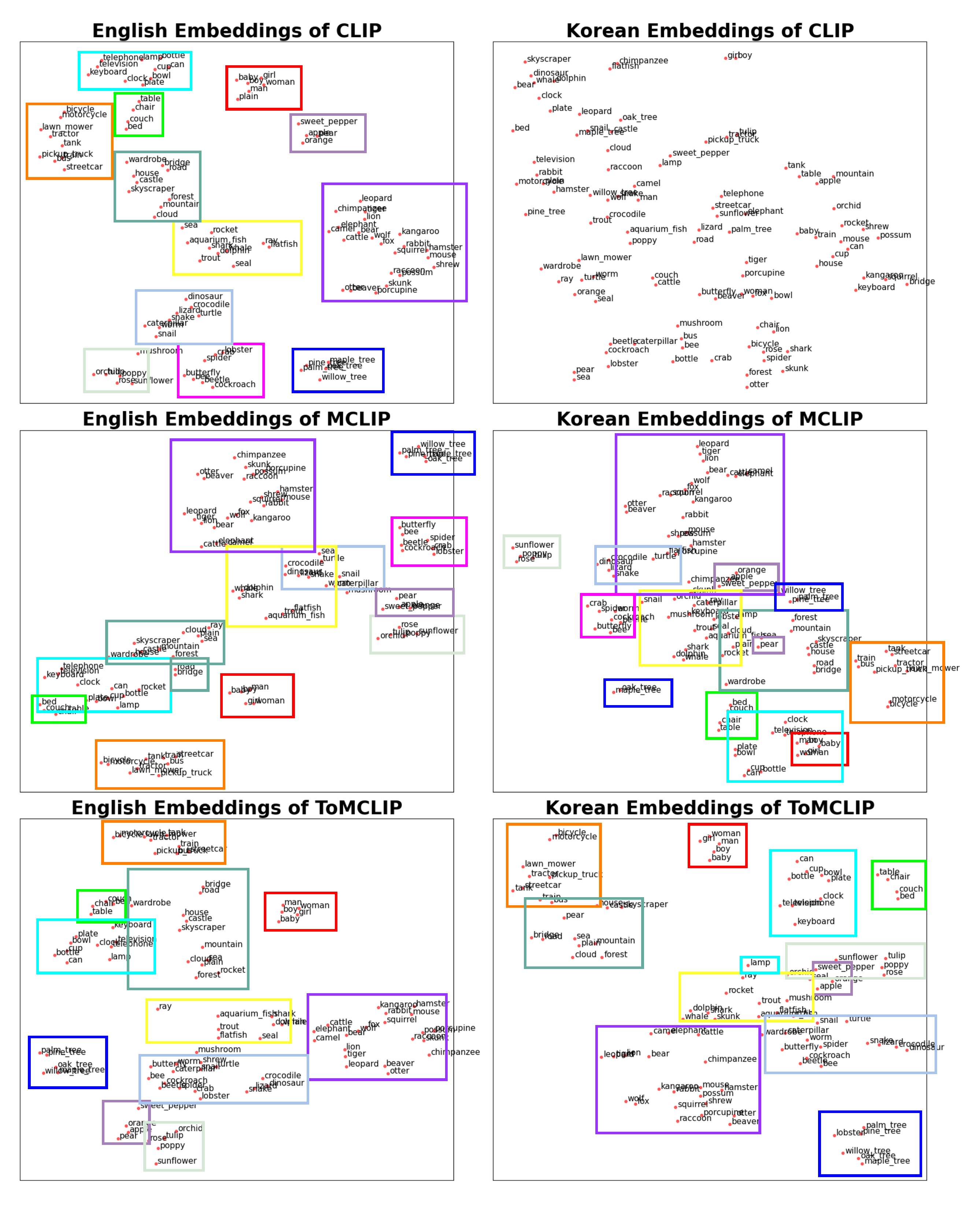}
  \caption{Two-dimensional t-SNE projections of English and Korean text embeddings
  }
  \label{fig:tsne}
\end{figure*}

\subsubsection{Evaluation on CIFAR-100}
We evaluate the zero-shot classification on CIFAR-100 to assess the alignment between the image and multilingual text embeddings. 
At inference, we use class-name prompts translated into 13 languages (e.g., \textit{``a photo of a \{class\}''}). 
Appendix~\ref{app:prompts} presents the complete prompt list. 
Table~\ref{tab:cifar100_lang_top10} reports the Top-10 accuracy (\%) per language (the Top-1 and Top-5 are provided in Tables~\ref{tab:cifar100_lang_top1} and~\ref{tab:cifar100_lang_top5} in Appendix~\ref{app:top-15_results}). 
In the \textbf{Full} setting, ToMCLIP surpasses MCLIP in all but one language (Polish, ``Pl''), yielding a higher average Top-10 accuracy overall (\(+0.88\)).
In the \textbf{Low} setting, ToMCLIP outperforms MCLIP across all 13 languages (\(+1.36\) on average). 
Note that En in the \textbf{Low} does not indicate catastrophic forgetting: CLIP’s text encoder is not used when evaluating (To)MCLIP.
Although MCLIP provides multilingual support, its cross-modal alignment remains suboptimal, whereas preserving the topological structure enables ToMCLIP to deliver more robust and consistent multilingual representations.
Table~\ref{tab:cifar100_multilingual_avg} summarizes the average Top-\(k\) (\(k\in\{1,5,10\}\)) accuracy.
The ToMCLIP performs the best for all \(k\) and both regimes. 
Among the ablations, ToMCLIP\((L_{\text{dm}})\) matches MCLIP, whereas ToMCLIP\((L_{\text{ta}})\) consistently improves upon MCLIP.
Using both losses together, ToMCLIP yields the strongest results. 
Adding \(L_{\text{dm}}\) on the baseline \(L_{\text{pw}}\) alone does not yield additional cross-modal alignment, whereas \(L_{\text{ta}}\) alone induces extra alignment and improves accuracy.
Nevertheless, \(L_{\text{dm}}\) is beneficial in conjunction with \(L_{\text{ta}}\), suggesting a complementary role that reinforces the alignment signal provided by \(L_{\text{ta}}\).
Appendices~\ref{app:batch_size} -~\ref{app:swd_K} presents the ablation studies on batch size, loss coefficients, homology dimension, graph sparsification threshold, and the number of SWD projections $K$, respectively.

\subsubsection{Topological Alignment Analysis}
\label{sec:analysis}

The topological alignment objective incorporated two loss components, $L_{ta}$ and $L_{dm}$. 
To assess their effects on the image-text latent space, we compare CLIP, MCLIP and ToMCLIP (trained with \(L_{\text{dm}}\) and \(L_{\text{ta}}\)). 
We use the same prompts for English (En) and Korean (Ko) derived from CIFAR-100 class labels as in the zero-shot evaluation (Appendix~\ref{app:prompts}). 

\paragraph{Visualization of Shared Latent Space.}
We visualize the sorted pairwise distance curves for En and Ko embeddings.
Figure~\ref{fig:pairwise_en_ko} displays distances sorted in ascending order and the dashed green line represents the absolute pairwise distance difference $|\mathrm{En} - \mathrm{Ko}|$.
In CLIP, a substantial discrepancy exists between the En and Ko distance distributions, reflected by a high $|\mathrm{En} - \mathrm{Ko}|$ curve because the CLIP model is trained using En caption datasets. 
The MCLIP, which is trained using multilingual data, exhibits improved alignment, significantly reducing the $|\mathrm{En} - \mathrm{Ko}|$ differences. 
Furthermore, ToMCLIP enhances the alignment, producing closer En and Ko curves.
These results visually confirm that ToMCLIP achieves the highest degree of cross-lingual geometric consistency in terms of the pairwise distance, suggesting that the topological alignment loss bridges language-induced gaps.

In addition to the distance curve analysis, we provide a visualization of the shared embedding space for En and Ko CIFAR-100 class label embeddings (Figure~\ref{fig:tsne}). 
Each point represents the embedding of a prompted class label, and distances between points reflect semantic relationships in the embedding space.
For each model (i.e., CLIP, MCLIP, and ToMCLIP), we project the embeddings in two dimensions using t-SNE~\citep{maaten2008visualizing} and highlight the class-level clusters. 
The bounding boxes were manually defined based on the En CLIP embeddings to capture coherent semantic groups, and the same grouping scheme was consistently applied to MCLIP and ToMCLIP for comparability. 
In CLIP, although the En embeddings form clear clusters, the Ko embeddings remain scattered, reflecting poor cross-lingual alignment. 
Moreover, MCLIP substantially improves alignment, with Ko embeddings aligning more closely to the manually defined clusters. 
Nevertheless, MCLIP still presents structural misalignment, as some clusters are mixed in the center.
Furthermore, the red box overlaps with neighboring groups, and the blue box is split into two subregions, which are clearly distinguished in the En embeddings. 
The ToMCLIP refines this structure, producing highly consistent cross-lingual clusters that preserve the semantic grouping. 
The red and blue clusters become well separated from other groups in En and Ko embeddings, highlighting the robustness of the topological alignment. 
Conbined with the distance curve, this visualization demonstrates that ToMCLIP minimizes pairwise distance discrepancies and preserves higher-level semantic structures across languages, providing complementary evidence for the effectiveness of the proposed topological alignment loss.

\begin{table}[!ht]
\centering
\caption{Mean and RMSE of $|\mathrm{En}-\mathrm{Ko}|$.}
\label{tab:pairwise-dist}
\resizebox{0.25\textwidth}{!}{%
\begin{tabular}{lcc}
\toprule
Model & Mean & RMSE \\
\midrule
CLIP   & 4.5238 & 4.5509 \\
MCLIP  & 0.3920 & 0.4081 \\
ToMCLIP & \textbf{0.3050} & \textbf{0.3133} \\
\bottomrule
\end{tabular}}
\end{table}

\begin{table*}[!ht]
\centering
\caption{Comparison of topological distances between English and Korean embeddings.}
\label{tab:topo_alignment}
\small
\resizebox{0.75\textwidth}{!}{%
\begin{tabular}{lcccccc}
\toprule
\multirow{2}{*}{Comparison} & \multirow{2}{*}{$W_2^{c}$} & \multicolumn{2}{c}{$W_2$} & \multicolumn{2}{c}{$\mathrm{SW}_2^{(50)}$} \\
\cmidrule(lr){3-4} \cmidrule(lr){5-6}
 & & 0-dim & 1-dim & 0-dim & 1-dim \\
\midrule
CLIP (En) vs. CLIP (Ko)       & 7.7870 & 34.5016 & 1.0468 & 2.8261 & 4.1593 \\
MCLIP (En) vs. MCLIP (Ko)      & 2.5988 & 5.1995 & 0.9250 & 0.3670 & 0.5964 \\
ToMCLIP (En) vs. ToMCLIP (Ko)    & \textbf{2.4929} & \textbf{4.2072} & \textbf{0.7444} & \textbf{0.3056} & \textbf{0.3393} \\
\bottomrule
\end{tabular}}
\end{table*}

\begin{table*}[!t]
\centering
\caption{Multilingual retrieval on xFlickr\&CO. Average R@k (\%) across 8 languages (Low vs. Full).
\textcolor{red}{$\blacktriangle$} indicates an improvement over MCLIP (same setting and direction),
\textcolor{blue}{$\blacktriangledown$} indicates a decrease.}
\label{tab:xflickrco_avg_low_full}
\resizebox{0.9\textwidth}{!}{%
\begin{tabular}{llllllll}
\toprule
\multirow{3}{*}{Direction} & \multirow{3}{*}{Model} & \multicolumn{3}{c}{\textbf{Low resource (1\% subset)}} & \multicolumn{3}{c}{\textbf{Full data (2M samples)}} \\
\cmidrule(lr){3-5} \cmidrule(lr){6-8}
 &  & R@1 & R@5 & R@10 & R@1 & R@5 & R@10 \\
\midrule
\multirow{5}{*}{IR} & CLIP & 12.08 & 22.12 & 27.19 & 12.08 & 22.12 & 27.19 \\
 & MCLIP & 33.51 & 62.04 & 73.70 & 50.13 & 77.51 & 85.86 \\
 & ToMCLIP($L_{\text{dm}}$) & 34.49\,\deltaup{0.98} & 62.93\,\deltaup{0.89} & \textbf{74.50}\,\deltaup{0.80} & \textbf{50.85}\,\deltaup{0.72} & \textbf{78.25}\,\deltaup{0.74} & \textbf{86.56}\,\deltaup{0.70} \\
 & ToMCLIP($L_{\text{ta}}$) & \textbf{34.50}\,\deltaup{0.99} & \textbf{62.96}\,\deltaup{0.93} & 74.45\,\deltaup{0.74} & 50.79\,\deltaup{0.66} & 78.01\,\deltaup{0.50} & 86.19\,\deltaup{0.33} \\
 & ToMCLIP & 34.03\,\deltaup{0.52} & 62.59\,\deltaup{0.56} & 74.00\,\deltaup{0.30} & 50.76\,\deltaup{0.63} & 77.99\,\deltaup{0.48} & 86.48\,\deltaup{0.62} \\
\midrule
\multirow{5}{*}{TR} & CLIP & 16.01 & 28.75 & 35.40 & 16.01 & 28.75 & 35.40 \\
 & MCLIP & 39.39 & 68.02 & 78.65 & 53.38 & 79.48 & 87.34 \\
 & ToMCLIP($L_{\text{dm}}$) & 39.71\,\deltaup{0.32} & 68.63\,\deltaup{0.61} & 79.38\,\deltaup{0.74} & 54.01\,\deltaup{0.63} & \textbf{80.38}\,\deltaup{0.90} & \textbf{88.08}\,\deltaup{0.74} \\
 & ToMCLIP($L_{\text{ta}}$) & \textbf{40.29}\,\deltaup{0.90} & \textbf{69.18}\,\deltaup{1.16} & \textbf{79.61}\,\deltaup{0.97} & 53.83\,\deltaup{0.45} & 79.91\,\deltaup{0.43} & 87.80\,\deltaup{0.46} \\
 & ToMCLIP & 39.51\,\deltaup{0.12} & 68.42\,\deltaup{0.40} & 78.96\,\deltaup{0.32} & \textbf{54.07}\,\deltaup{0.69} & 79.98\,\deltaup{0.50} & 87.67\,\deltaup{0.33} \\
\bottomrule
\end{tabular}}
\end{table*}

\paragraph{Quantitative Analysis of Shared Latent Space.}

The $L_{dm}$ term minimizes the MSE between two pairwise distance matrices. 
Table~\ref{tab:pairwise-dist} reports the mean and RMSE of the absolute sorted pairwise distance differences ($|\mathrm{En}-\mathrm{Ko}|$) between En and Ko embeddings.
The proposed ToMCLIP achieves substantially lower values than MCLIP, indicating improved alignment.

The $L_{\text{ta}}$ promotes topological consistency by minimizing the distance between the persistence diagrams of the two embedding sets. 
By the stability inequality (Eq.~\ref{eq:stability}), decreasing $L_{\text{ta}}$ tightens a certified lower bound on the $p$-Wasserstein distance between the corresponding point clouds. 
To verify that minimizing $L_{\text{ta}}$ yields lower $W_p^{\text{c}}$, Table~\ref{tab:topo_alignment} reports three metrics: $W_2^{\text{c}}$ (2-Wasserstein between the raw embeddings), $W_2$ (2-Wasserstein between the persistence diagrams), and $\mathrm{SW}_2^{(50)}$ (sliced 2-Wasserstein between the  persistence diagrams using 50 projections).
Overall, ToMCLIP yields the lowest cross-lingual distances across all metrics, confirming that topology-aware training with $L_{\text{ta}}$ enhances the topological alignment.

\subsubsection{Multilingual Image-Text Retrieval on xFlickr\&CO}
\label{sec:retrieval_xflickrco}

This work evaluates multilingual image-text retrieval on xFlickr\&CO~\citep{bugliarello-etal-2022-iglue} across eight languages (En, Es, De, Id, Ru, Tr, Ja, Zh). The benchmark comprises \(2\text{K}\) images (\(1\text{K}\) from Flickr30K and \(1\text{K}\) from MSCOCO), each paired with a single parallel caption in all eight languages, enabling evaluation of both retrieval directions. 
This work presents the results for image retrieval (\textbf{IR}; text\(\rightarrow\)image) and text retrieval (\textbf{TR}; image\(\rightarrow\)text). Following standard practice, we compute recall at $K$ (R@K, \(K\in\{1,5,10\}\)) and average the scores across languages. 
For each language, R@K is evaluated over \(2{,}000\) queries.

Table~\ref{tab:xflickrco_avg_low_full} summarizes the language-averaged results and Table~\ref{tab:xflickrco_top1} (Appendix~\ref{app:retrieval_xflickrco}) breaks down R@1 by language.
In both tables, red \textcolor{red}{$\blacktriangle$} (\ blue \textcolor{blue}{$\blacktriangledown$}) marks improvement (degradation) over MCLIP under the same settings and direction. In the \textbf{Full} regime, \mbox{ToMCLIP($L_{\text{dm}}$)}, \mbox{ToMCLIP($L_{\text{ta}}$)}, and \mbox{ToMCLIP} yield consistent average gains over MCLIP for \textbf{IR} and \textbf{TR} across all metrics (R@1,5, and 10). In the more challenging \textbf{Low} regime, they also achieve consistent average gains over MCLIP. These results indicate that the proposed losses improve cross-lingual alignment in the shared embedding space.



\section{CONCLUSION}
This work introduces ToMCLIP, a topology-aware alignment framework for multilingual contrastive VLMs, augmenting instance-level matching with topology-preserving objectives. 
The ToMCLIP improves the zero-shot CIFAR-100 performance, and stronger multilingual retrieval performance on the xFlickr\&CO.
Furthermore, ToMCLIP enhances the structural coherence of the shared embedding space.
Beyond multilingual alignment, the topological alignment loss provides a general objective for aligning embedding spaces, encompassing cross-modal alignment, knowledge distillation, and dimensionality reduction. 
While our experiments focus on $H_0$ features and single-run evaluations, extending to higher-order topological signals and more comprehensive statistical validation remains an important direction for future work.

\subsubsection*{Acknowledgements}
This paper is supported by Basic Science Research Institute Fund, whose NRF grant number is RS -2021-NR060139.

\bibliography{sections/reference}

\begin{thebibliography}{}

\bibitem[Aggarwal, 2019]{param_aggarwal_2019}
Aggarwal, P. (2019).
\newblock Fashion product images dataset.

\bibitem[Alayrac et~al., 2022]{alayrac2022flamingo}
Alayrac, J.-B., Donahue, J., Luc, P., Miech, A., Barr, I., Hasson, Y., Lenc, K., Mensch, A., Millican, K., Reynolds, M., et~al. (2022).
\newblock Flamingo: a visual language model for few-shot learning.
\newblock {\em Advances in Neural Information Processing Systems}, 35:23716--23736.

\bibitem[Bai et~al., 2023]{bai2023qwen}
Bai, J., Bai, S., Yang, S., Wang, S., Tan, S., Wang, P., Lin, J., Zhou, C., and Zhou, J. (2023).
\newblock Qwen-vl: A versatile vision-language model for understanding, localization, text reading, and beyond.
\newblock {\em arXiv preprint arXiv:2308.12966}.

\bibitem[Beyer et~al., 2024]{beyer2024paligemma}
Beyer, L., Steiner, A., Pinto, A.~S., Kolesnikov, A., Wang, X., Salz, D., Neumann, M., Alabdulmohsin, I., Tschannen, M., Arnab, A., et~al. (2024).
\newblock Paligemma: A versatile 3b vlm for transfer.
\newblock {\em arXiv preprint arXiv:2407.07726}.

\bibitem[Bonneel et~al., 2015]{bonneel2015sliced}
Bonneel, N., Rabin, J., Peyr{\'e}, G., and Pfister, H. (2015).
\newblock Sliced and radon wasserstein barycenters of measures.
\newblock {\em Journal of Mathematical Imaging and Vision}, 51(1):22--45.

\bibitem[Bordes et~al., 2024]{bordes2024introduction}
Bordes, F., Pang, R.~Y., Ajay, A., Li, A.~C., Bardes, A., Petryk, S., Ma{\~n}as, O., Lin, Z., Mahmoud, A., Jayaraman, B., et~al. (2024).
\newblock An introduction to vision-language modeling.
\newblock {\em arXiv preprint arXiv:2405.17247}.

\bibitem[Bugliarello et~al., 2022]{bugliarello-etal-2022-iglue}
Bugliarello, E., Liu, F., Pfeiffer, J., Reddy, S., Elliott, D., Ponti, E.~M., and Vuli{\'c}, I. (2022).
\newblock {IGLUE}: A benchmark for transfer learning across modalities, tasks, and languages.
\newblock In Chaudhuri, K., Jegelka, S., Song, L., Szepesvari, C., Niu, G., and Sabato, S., editors, {\em Proceedings of the 39th International Conference on Machine Learning}, volume 162 of {\em Proceedings of Machine Learning Research}, pages 2370--2392. PMLR.

\bibitem[Carlsson et~al., 2022]{carlsson2022cross}
Carlsson, F., Eisen, P., Rekathati, F., and Sahlgren, M. (2022).
\newblock Cross-lingual and multilingual clip.
\newblock In {\em Proceedings of the thirteenth language resources and evaluation conference}, pages 6848--6854.

\bibitem[Carri{\`e}re et~al., 2020]{carriere2020perslay}
Carri{\`e}re, M., Chazal, F., Ike, Y., Lacombe, T., Royer, M., and Umeda, Y. (2020).
\newblock Perslay: A neural network layer for persistence diagrams and new graph topological signatures.
\newblock In {\em International Conference on Artificial Intelligence and Statistics}, pages 2786--2796. PMLR.

\bibitem[Chen et~al., 2023a]{chen2023mclip}
Chen, G., Hou, L., Chen, Y., Dai, W., Shang, L., Jiang, X., Liu, Q., Pan, J., and Wang, W. (2023a).
\newblock mclip: Multilingual clip via cross-lingual transfer.
\newblock In {\em Proceedings of the 61st Annual Meeting of the Association for Computational Linguistics (Volume 1: Long Papers)}, pages 13028--13043.

\bibitem[Chen et~al., 2023b]{chen2023altclip}
Chen, Z., Liu, G., Zhang, B.-W., Yang, Q., and Wu, L. (2023b).
\newblock Altclip: Altering the language encoder in clip for extended language capabilities.
\newblock In {\em Findings of the Association for Computational Linguistics: ACL 2023}, pages 8666--8682.

\bibitem[Chen et~al., 2024]{chen2024internvl}
Chen, Z., Wu, J., Wang, W., Su, W., Chen, G., Xing, S., Zhong, M., Zhang, Q., Zhu, X., Lu, L., et~al. (2024).
\newblock Internvl: Scaling up vision foundation models and aligning for generic visual-linguistic tasks.
\newblock {\em arXiv preprint arXiv:2312.14238}.

\bibitem[Cherti et~al., 2023]{cherti2023reproducible}
Cherti, M., Beaumont, R., Wightman, R., Wortsman, M., Ilharco, G., Gordon, C., Schuhmann, C., Schmidt, L., and Jitsev, J. (2023).
\newblock Reproducible scaling laws for contrastive language-image learning.
\newblock In {\em Proceedings of the IEEE/CVF conference on computer vision and pattern recognition}, pages 2818--2829.

\bibitem[Conneau et~al., 2019]{conneau2019unsupervised}
Conneau, A., Khandelwal, K., Goyal, N., Chaudhary, V., Wenzek, G., Guzm{\'a}n, F., Grave, E., Ott, M., Zettlemoyer, L., and Stoyanov, V. (2019).
\newblock Unsupervised cross-lingual representation learning at scale.
\newblock {\em arXiv preprint arXiv:1911.02116}.

\bibitem[Cormen et~al., 2022]{cormen2022introduction}
Cormen, T.~H., Leiserson, C.~E., Rivest, R.~L., and Stein, C. (2022).
\newblock {\em Introduction to algorithms}.
\newblock MIT press.

\bibitem[Hai et~al., 2025]{hai2025topology}
Hai, L.~T., Le, T.~D., Ding, Z., Tian, Q., and Hy, T.-S. (2025).
\newblock Topology-guided knowledge distillation for efficient point cloud processing.
\newblock {\em arXiv preprint arXiv:2505.08101}.

\bibitem[Huang, 2025]{huang2025topology}
Huang, D. (2025).
\newblock Topology-aware clip few-shot learning.
\newblock {\em arXiv preprint arXiv:2505.01694}.

\bibitem[Jia et~al., 2021]{jia2021scaling}
Jia, C., Yang, Y., Xia, Y., Chen, Y.-T., Parekh, Z., Pham, H., Le, Q., Sung, Y.-H., Li, Z., and Duerig, T. (2021).
\newblock Scaling up visual and vision-language representation learning with noisy text supervision.
\newblock In {\em International conference on machine learning}, pages 4904--4916. PMLR.

\bibitem[Kim et~al., 2024]{kim2024topological}
Kim, J., You, J., Lee, D., Kim, H.~Y., and Jung, J.-H. (2024).
\newblock Do topological characteristics help in knowledge distillation?
\newblock In {\em Forty-first International Conference on Machine Learning}.

\bibitem[Kruskal, 1956]{kruskal1956shortest}
Kruskal, J.~B. (1956).
\newblock On the shortest spanning subtree of a graph and the traveling salesman problem.
\newblock {\em Proceedings of the American Mathematical society}, 7(1):48--50.

\bibitem[Li et~al., 2023a]{li2023blip2}
Li, J., Li, D., Savarese, S., and Hoi, S. (2023a).
\newblock Blip-2: Bootstrapping language-image pre-training with frozen image encoders and large language models.
\newblock {\em International Conference on Machine Learning}, pages 19730--19742.

\bibitem[Li et~al., 2023b]{li2023scaling}
Li, Y., Fan, H., Hu, R., Feichtenhofer, C., and He, K. (2023b).
\newblock Scaling language-image pre-training via masking.
\newblock In {\em Proceedings of the IEEE/CVF Conference on Computer Vision and Pattern Recognition}, pages 23390--23400.

\bibitem[Li et~al., 2022]{li2022supervision}
Li, Y., Liang, F., Zhao, L., Cui, Y., Ouyang, W., Shao, J., Yu, F., and Yan, J. (2022).
\newblock Supervision exists everywhere: A data efficient contrastive language-image pre-training paradigm.
\newblock In {\em International Conference on Learning Representations}.

\bibitem[Liu et~al., 2024a]{liu2024improved}
Liu, H., Li, C., Li, Y., and Lee, Y.~J. (2024a).
\newblock Improved baselines with visual instruction tuning.
\newblock {\em arXiv preprint arXiv:2310.03744}.

\bibitem[Liu et~al., 2024b]{liu2024llavanext}
Liu, H., Li, C., Li, Y., Li, B., Zhang, Y., Shen, S., and Lee, Y.~J. (2024b).
\newblock Llava-next: Improved reasoning, ocr, and world knowledge.
\newblock {\em https://llava-vl.github.io/blog/2024-01-30-llava-next/}.

\bibitem[Liu et~al., 2024c]{liu2024visual}
Liu, H., Li, C., Wu, Q., and Lee, Y.~J. (2024c).
\newblock Visual instruction tuning.
\newblock {\em Advances in neural information processing systems}, 36.

\bibitem[Maaten and Hinton, 2008]{maaten2008visualizing}
Maaten, L. v.~d. and Hinton, G. (2008).
\newblock Visualizing data using t-sne.
\newblock {\em Journal of machine learning research}, 9(Nov):2579--2605.

\bibitem[Moor et~al., 2020]{moor2020topological}
Moor, M., Horn, M., Rieck, B., and Borgwardt, K. (2020).
\newblock Topological autoencoders.
\newblock In {\em International conference on machine learning}, pages 7045--7054. PMLR.

\bibitem[OpenAI, 2023]{openai2023gpt4v}
OpenAI (2023).
\newblock Gpt-4v(ision) system card.
\newblock {\em https://openai.com/research/gpt-4v-system-card}.

\bibitem[Otter et~al., 2017]{otter2017roadmap}
Otter, N., Porter, M.~A., Tillmann, U., Grindrod, P., and Harrington, H.~A. (2017).
\newblock A roadmap for the computation of persistent homology.
\newblock {\em EPJ Data Science}, 6(1):17.

\bibitem[Papillon et~al., 2023]{papillon2023architectures}
Papillon, M., Sanborn, S., Hajij, M., and Miolane, N. (2023).
\newblock Architectures of topological deep learning: A survey of message-passing topological neural networks.
\newblock {\em arXiv preprint arXiv:2304.10031}.

\bibitem[Radford et~al., 2021]{radford2021learning}
Radford, A., Kim, J.~W., Hallacy, C., Ramesh, A., Goh, G., Agarwal, S., Sastry, G., Askell, A., Mishkin, P., Clark, J., et~al. (2021).
\newblock Learning transferable visual models from natural language supervision.
\newblock In {\em International conference on machine learning}, pages 8748--8763. PMLR.

\bibitem[Rahim et~al., 2024]{rahim2024topological}
Rahim, A.~A. et~al. (2024).
\newblock Topological perspectives on optimal multimodal embedding spaces.
\newblock {\em arXiv preprint arXiv:2405.18867}.

\bibitem[Schuhmann et~al., 2021]{schuhmann2021laion}
Schuhmann, C., Vencu, R., Beaumont, R., Kaczmarczyk, R., Mullis, C., Katta, A., Coombes, T., Jitsev, J., and Komatsuzaki, A. (2021).
\newblock Laion-400m: Open dataset of clip-filtered 400 million image-text pairs.
\newblock {\em arXiv preprint arXiv:2111.02114}.

\bibitem[Skraba and Turner, 2020]{skraba2020wasserstein}
Skraba, P. and Turner, K. (2020).
\newblock Wasserstein stability for persistence diagrams.
\newblock {\em arXiv preprint arXiv:2006.16824}.

\bibitem[Tarjan, 1979]{tarjan1979class}
Tarjan, R.~E. (1979).
\newblock A class of algorithms which require nonlinear time to maintain disjoint sets.
\newblock {\em Journal of computer and system sciences}, 18(2):110--127.

\bibitem[Team et~al., 2023]{team2023gemini}
Team, G., Anil, R., Borgeaud, S., Wu, Y., Alayrac, J.-B., Yu, J., Soricut, R., Schalkwyk, J., Dai, A.~M., Hauth, A., et~al. (2023).
\newblock Gemini: a family of highly capable multimodal models.
\newblock {\em arXiv preprint arXiv:2312.11805}.

\bibitem[Team et~al., 2024]{team2024gemma2}
Team, G., Mesnard, T., Hardin, C., Dadashi, R., Bhupatiraju, S., Pathak, S., Sifre, L., Schünemann, M., Rivière, M., Kale, M.~S., et~al. (2024).
\newblock Gemma 2: Improving open language models at a practical size.
\newblock {\em arXiv preprint arXiv:2408.00118}.

\bibitem[Team, 2024]{qwen2024qwen25vl}
Team, Q. (2024).
\newblock Qwen2.5-vl: The latest qwen multimodal large language model.
\newblock {\em https://qwenlm.github.io/blog/qwen2.5-vl/}.

\bibitem[Trofimov et~al., 2023]{trofimov2023learning}
Trofimov, I., Cherniavskii, D., Tulchinskii, E., Balabin, N., Burnaev, E., and Barannikov, S. (2023).
\newblock Learning topology-preserving data representations.
\newblock {\em arXiv preprint arXiv:2302.00136}.

\bibitem[Wang et~al., 2024a]{wang2024qwen2vl}
Wang, P., Bai, S., Tan, S., Wang, S., Fan, Z., Bai, J., Chen, K., Liu, X., Wang, J., Ge, W., et~al. (2024a).
\newblock Qwen2-vl: Enhancing vision-language model's perception of the world at any resolution.
\newblock {\em arXiv preprint arXiv:2409.12191}.

\bibitem[Wang et~al., 2024b]{wang2024persistence}
Wang, Y., Zhu, P., Chen, D., and Hu, Q. (2024b).
\newblock Persistence homology distillation for semi-supervised continual learning.
\newblock {\em Advances in Neural Information Processing Systems}, 37:76332--76355.

\bibitem[Wen et~al., 2024]{wen2024tensor}
Wen, T., Chen, E., and Chen, Y. (2024).
\newblock Tensor-view topological graph neural network.
\newblock In {\em International Conference on Artificial Intelligence and Statistics}, pages 4330--4338. PMLR.

\bibitem[Yang et~al., 2024]{yang2024embracing}
Yang, B., Dai, Y., Cheng, X., Li, Y., Raza, A., and Zou, Y. (2024).
\newblock Embracing language inclusivity and diversity in clip through continual language learning.
\newblock In {\em Proceedings of the AAAI Conference on Artificial Intelligence}, pages 6458--6466.

\bibitem[Yao et~al., 2022]{yao2022filip}
Yao, L., Huang, R., Hou, L., Lu, G., Niu, M., Xu, H., Liang, X., Li, Z., Jiang, X., and Xu, C. (2022).
\newblock Filip: Fine-grained interactive language-image pre-training.
\newblock In {\em International Conference on Learning Representations}.

\bibitem[Young et~al., 2024]{young2024yi}
Young, A., Chen, B., Li, C., Huang, C., Zhang, G., Zhang, G., Li, H., Zhu, J., Chen, J., Chang, J., et~al. (2024).
\newblock Yi: Open foundation models by 01.ai.
\newblock {\em arXiv preprint arXiv:2403.04652}.

\bibitem[Zhang et~al., 2024]{zhang2024homology}
Zhang, H., Zhang, L., Zhang, Y., and Mao, Z. (2024).
\newblock Homology consistency constrained efficient tuning for vision-language models.
\newblock {\em Advances in Neural Information Processing Systems}, 37:93011--93032.

\bibitem[Zilberstein et~al., 2024]{zilberstein2024topology}
Zilberstein, N., Malhotra, A., Hamidi-Rad, S., and Deenoo, Y. (2024).
\newblock Topology preserving regularization for independent training of inter-operable models.
\newblock In {\em UniReps: 2nd Edition of the Workshop on Unifying Representations in Neural Models}.

\end{thebibliography}

\section*{Checklist}



\begin{enumerate}

  \item For all models and algorithms presented, check if you include:
  \begin{enumerate}
    \item A clear description of the mathematical setting, assumptions, algorithm, and/or model. 
    [Yes, see Section~\ref{sec:topological_alignment}.]
    \item An analysis of the properties and complexity (time, space, sample size) of any algorithm. [Yes, see Section~\ref{sec:approximate_PD} and~\ref{sec:conn_spars}.]
    \item (Optional) Anonymized source code, with specification of all dependencies, including external libraries. [Yes, see zipped file attached.]
  \end{enumerate}

  \item For any theoretical claim, check if you include:
  \begin{enumerate}
    \item Statements of the full set of assumptions of all theoretical results. [Yes, see Section~\ref{sec:approximate_PD}.]
    \item Complete proofs of all theoretical results. [Yes, see Appendix~\ref{app:proof}.]
    \item Clear explanations of any assumptions. [Yes, see Section~\ref{sec:approximate_PD}.]     
  \end{enumerate}

  \item For all figures and tables that present empirical results, check if you include:
  \begin{enumerate}
    \item The code, data, and instructions needed to reproduce the main experimental results (either in the supplemental material or as a URL). [Yes, see Appendix~\ref{app:data}.]
    \item All the training details (e.g., data splits, hyperparameters, how they were chosen). [Yes, see Appendix~\ref{app:data}.]
    \item A clear definition of the specific measure or statistics and error bars (e.g., with respect to the random seed after running experiments multiple times). 
    [In the \textbf{Full} setting, results are from a single run following previous work~\citep{radford2021learning,carlsson2022cross,chen2023mclip}. However, in the \textbf{Low} setting, we report the mean over 3 independent runs.]
    \item A description of the computing infrastructure used. (e.g., type of GPUs, internal cluster, or cloud provider). [Yes, see Appendix~\ref{app:training_time}]
  \end{enumerate}

  \item If you are using existing assets (e.g., code, data, models) or curating/releasing new assets, check if you include:
  \begin{enumerate}
    \item Citations of the creator If your work uses existing assets. [Yes, see Appendix~\ref{app:data}.]
    \item The license information of the assets, if applicable. [Not Applicable]
    \item New assets either in the supplemental material or as a URL, if applicable. [Not Applicable]
    \item Information about consent from data providers/curators. [Not Applicable]
    \item Discussion of sensible content if applicable, e.g., personally identifiable information or offensive content. [Not Applicable]
  \end{enumerate}

  \item If you used crowdsourcing or conducted research with human subjects, check if you include:
  \begin{enumerate}
    \item The full text of instructions given to participants and screenshots. [Not Applicable]
    \item Descriptions of potential participant risks, with links to Institutional Review Board (IRB) approvals if applicable. [Not Applicable]
    \item The estimated hourly wage paid to participants and the total amount spent on participant compensation. [Not Applicable]
  \end{enumerate}

\end{enumerate}

\clearpage
\appendix
\thispagestyle{empty}

\onecolumn
\aistatstitle{Supplementary Materials: Appendices}


\setcounter{tocdepth}{3}
\tableofcontents 

\section{RELATED WORKS}
\label{app:relatedworks}

\subsection{Contrastive Vision-Language Models}
\label{app:contrastive-vlms}

Contrastive vision-language models (VLMs) learn joint representations of images and text by maximizing the similarity between matched pairs while minimizing it for unmatched pairs. CLIP (Contrastive Language-Image Pre-training)~\citep{radford2021learning} pioneered this approach by training dual encoders on 400 million image-text pairs collected from the internet. The model employs a symmetric cross-entropy loss over the similarity matrix of image and text embeddings within each batch, enabling zero-shot transfer to downstream tasks without task-specific fine-tuning.

ALIGN~\citep{jia2021scaling} scaled this approach further by leveraging a noisy dataset of over one billion image-text pairs, demonstrating that the noise in web-scraped data can be overcome with sufficient scale. Unlike CLIP, which uses curated data, ALIGN shows that raw alt-text data can be effective when combined with a simple dual-encoder architecture and contrastive learning objective.

Several subsequent works have improved upon these foundations. FLIP~\citep{li2023scaling} introduced a masking strategy during training to reduce computational costs while maintaining performance. DeCLIP~\citep{li2022supervision} enhanced data efficiency through self-supervised learning and nearest-neighbor supervision. FILIP~\citep{yao2022filip} improved fine-grained alignment by introducing token-wise maximum similarity between image patches and text tokens.

The key advantages of contrastive models include: (1) computational efficiency during inference, as image and text encoders can be cached and indexed separately; (2) flexibility in swapping encoders for different modalities or languages; and (3) strong performance on retrieval tasks. These properties make contrastive models particularly suitable for multilingual extensions, as the text encoder can be replaced or fine-tuned for different languages while keeping the image encoder fixed.

Despite their success, contrastive models face challenges in maintaining consistency across languages when extended to multilingual settings, particularly in preserving the geometric structure of the shared embedding space. 
We address this limitation through topological alignment.

\subsection{Autoregressive Multimodal Large Language Models}
\label{app:autoregressive-llms}

While our work focuses on contrastive VLMs, we briefly review recent autoregressive multimodal Large Language Models (LLMs) to contextualize our approach within the broader landscape of vision-language understanding. Unlike contrastive models that learn aligned embedding spaces, autoregressive multimodal LLMs generate text conditioned on visual inputs through next-token prediction.

Flamingo~\citep{alayrac2022flamingo} pioneered the frozen LLM approach by introducing cross-attention layers between a pretrained vision encoder and language model, enabling few-shot learning on vision-language tasks. BLIP-2~\citep{li2023blip2} proposed Q-Former, a lightweight module that bridges frozen image encoders and LLMs through a set of learnable query tokens, significantly reducing training costs while achieving strong performance.

LLaVA~\citep{liu2024visual} demonstrated that visual instruction tuning (training on instruction-following data in the visual domain) can produce capable multimodal assistants. The model uses a simple projection layer to connect CLIP visual features with an LLM, showing that architectural simplicity combined with high-quality instruction data can be highly effective. Subsequent versions like LLaVA-1.5~\citep{liu2024improved} and LLaVA-NeXT~\citep{liu2024llavanext} have improved resolution handling and reasoning capabilities.

Commercial models have pushed the boundaries further. GPT-4V~\citep{openai2023gpt4v} demonstrates unprecedented visual understanding and reasoning, though architectural details remain proprietary. Gemini~\citep{team2023gemini} achieves state-of-the-art performance across numerous multimodal benchmarks through native multimodal pretraining rather than connecting separate vision and language models.

The Qwen series has emerged as a particularly strong line of multimodal models. Qwen-VL~\citep{bai2023qwen} introduced a versatile VLM supporting multiple languages and resolutions. Qwen2-VL~\citep{wang2024qwen2vl} significantly improved upon this with enhanced visual reasoning, video understanding, and multilingual OCR capabilities across 29 languages. The latest Qwen2.5-VL~\citep{qwen2024qwen25vl} further advances the architecture with dynamic resolution support and improved instruction following, achieving state-of-the-art performance on various benchmarks while maintaining efficient inference.

Similarly, Google's Gemma family has expanded into multimodal territory. PaliGemma~\citep{beyer2024paligemma} combines a SigLIP vision encoder with Gemma language models for versatile vision-language understanding. Gemma-2~\citep{team2024gemma2} improved the base architecture, leading to enhanced multimodal capabilities when combined with vision encoders. These models demonstrate strong performance while being more accessible than larger commercial offerings.

Other notable open-source alternatives include InternVL~\citep{chen2024internvl}, which scales vision foundation models for generic visual-linguistic tasks, and the Yi-VL series~\citep{young2024yi}, which offers competitive performance with bilingual (Chinese-English) specialization.

These autoregressive models excel at complex reasoning, visual question answering, and generating detailed descriptions. However, they require significant computational resources during inference due to sequential token generation and cannot easily cache embeddings for retrieval tasks. Furthermore, their multilingual capabilities typically depend on the underlying LLM's language coverage, making it challenging to add new languages without extensive retraining.

The fundamental architectural differences between contrastive and autoregressive approaches lead to complementary strengths: contrastive models like CLIP excel at retrieval and classification with efficient inference, while autoregressive models provide superior reasoning and generation capabilities at higher computational cost. Our topology-aware alignment method specifically targets the unique challenges of multilingual contrastive models, where maintaining geometric consistency across languages is crucial for retrieval performance.

\section{PERSISTENT HOMOLOGY}
\label{app:ph}

Topological data analysis (TDA) characterizes the shape of data by extracting topological features that are stable to small perturbations. We assume the observed points are sampled from an unknown manifold embedded in a metric space. Given a finite point cloud \(X=\{x_i\}_{i=1}^N\) with metric \(d\), we construct a nested family of simplicial complexes (e.g., a Vietoris-Rips filtration) indexed by a scale parameter \(\alpha\). Persistent homology computes homology across scales and records when features, such as connected components and loops, are born and die. The resulting multiset of birth-death pairs is the persistence diagram. These summaries provide geometric signals.

\paragraph{Point Clouds and the Vietoris-Rips Filtration.}
Let $X=\{x_i\}_{i=1}^N \subset (\mathcal{X},d)$.
For $\alpha \ge 0$, the \emph{Vietoris-Rips (Rips) complex $\mathrm{VR}_\alpha(X)$} is the abstract simplicial complex whose $k$-simplices are all $(k{+}1)$-tuples $\{x_{i_0},\ldots,x_{i_k}\}$ with pairwise distances $\max\limits_{p,q} d(x_{i_p},x_{i_q}) \le \alpha$.
As $\alpha$ increases, the complexes are nested
\begin{equation}
    \mathrm{VR}_{\alpha_1}(X)\;\subseteq\;\mathrm{VR}_{\alpha_2}(X)\quad\text{for }\alpha_1\le \alpha_2,    
\end{equation}
yielding the Rips filtration $\{\mathrm{VR}_\alpha(X)\}_{\alpha\ge 0}$.

\paragraph{Weighted Graphs and the Rips Filtration.}
For a weighted graph $G = (V,E,w)$ with weight function $\omega:E\to\mathbb{R}_{\ge0}$, 
we define the \emph{Rips complex $\mathrm{VR}_\alpha(G)$} as the abstract simplicial complex whose 
1-skeleton consists of the vertex set $V$ and all edges $(u,v)\in E$ with $w(u,v)\le \alpha$.
Higher-order simplices are then included whenever all their edges are present.
As $\alpha$ increases, the complexes form a nested sequence 
$\mathrm{VR}_{\alpha_1}(G)\subseteq \mathrm{VR}_{\alpha_2}(G)$ for $\alpha_1\le\alpha_2$,
yielding the Rips filtration $\{\mathrm{VR}_\alpha(G)\}_{\alpha\ge0}$ induced by the graph weights.

\paragraph{Persistent Homology and Persistence Diagrams.}
Fix a homological dimension $k\in\{0,1,2,\ldots\}$ and a coefficient field (we use $\mathbb{Z}_2$).
The inclusion maps in the filtration induce homomorphism between homology groups $H_k(\mathrm{VR}_{\alpha_1}) \to H_k(\mathrm{VR}_{\alpha_2})$ for $\alpha_1\le\alpha_2$.
Each topological feature $\eta$ (a $k$-dimensional class) \emph{appears} (is born) at scale $b$ ($H_k(\mathrm{VR}_{b})$) and \emph{disappears} (dies) at scale $d\ge b$ ($H_k(\mathrm{VR}_{d})$).
The multiset of pairs $(b,d)$ is the $k$-dimensional \emph{persistence diagram} $D_k$.
For $k{=}0$, all components are born at $b{=}0$, and deaths record the merger times of components.

\paragraph{Distances Between Persistence Diagrams.}
Let $D_1$ and $D_2$ be persistence diagrams, and let $\Delta=\{(t,t)\,:\,t\in\mathbb{R}\}$ be the diagonal line in $\mathbb{R}^2$.
We compare diagrams by allowing matches to points on $\Delta$.
For $p\in[1,\infty)$, the \emph{$p$-Wasserstein distance} is
\begin{equation}
    W_p(D_1,D_2)
    \;=\;
    \Big[
    \inf_{\gamma}\;\sum_{u\in D_1\cup \Delta}
    (\lVert u-\gamma(u)\rVert_{p})^{\,p}
    \Big]^{\!1/p},
\end{equation}
where $\gamma$ ranges over all bijections between $D_1 \cup \Delta$ and $D_2 \cup \Delta$, and $\lVert \cdot \rVert_p$ denotes $L_p$-norm.
The special case $p{=}\infty$ yields the \emph{bottleneck distance}
\begin{equation}
    W_\infty(D_1,D_2)
    \;=\;
    \inf_{\gamma}\;\sup_{u\in D_1\cup \Delta}\lVert u-\gamma(u)\rVert_\infty.    
\end{equation}
These metrics enjoy well-known stability properties: small perturbations of the input metric (or filtration function) produce small changes in the diagrams~\citep{skraba2020wasserstein}.

\paragraph{Sliced Wasserstein distance (SWD).} SWD approximates the $d$-dimensional Wasserstein distance by projecting the data onto many 1-dimensional lines and averaging the resulting one-dimensional Wasserstein costs. 
This yields a fast $\mathcal{O}(K\,N\log N)$, differentiable, and GPU-friendly objective that is well suited as a training loss. 
We now give the formal definition.

Given two finite point sets $X=\{x_i\}_{i=1}^{N}\subset\mathbb{R}^{n}$ and $Y=\{y_j\}_{j=1}^{N}\subset\mathbb{R}^{n}$ (uniform weights), the sliced $p$-Wasserstein distance compares them by averaging one-dimensional $p$-Wasserstein costs of their projections.
For a unit direction $\theta\in S^{n-1}$, project $s_i=\langle x_i,\theta\rangle$ and $t_j=\langle y_j,\theta\rangle$, and let $s_{(1)}\le\cdots\le s_{(N)}$ and $t_{(1)}\le\cdots\le t_{(N)}$ be the sorted values.
The 1D cost along $\theta$ is
\[
W_p^{\mathrm{1D}}(\theta)
=\Big(\frac{1}{N}\sum_{i=1}^{N}\big|\,s_{(i)}-t_{(i)}\,\big|^{p}\Big)^{\!1/p}.
\]
Averaging over directions yields
\[
\mathrm{SW}_p(X,Y)
=\Bigg(\int_{S^{d-1}}\!\!\big(W_p^{\mathrm{1D}}(\theta)\big)^{p}\,d\sigma(\theta)\Bigg)^{\!1/p}
\]
where $\sigma$ is the uniform measure on $S^{d-1}$.
In practice, we approximate the integral with $K$ directions $\{\theta_k\}_{k=1}^{K}$ sampled uniformly:
\begin{equation}
\label{eq:swd}
    \mathrm{SW}_p^{(K)}(X,Y)
    =\Big(\tfrac{1}{K}\sum_{k=1}^{K}\big(W_p^{\mathrm{1D}}(\theta_k)\big)^{p}\Big)^{\!1/p}
\end{equation}
which can be computed in $\mathcal{O}(K\,N\log N)$ time via sorting per direction.

\section{NOTATIONS}
\label{app:notations}

\begin{table}[t]
\centering
\small
\caption{Notations used in Section~\ref{sec:topological_alignment} (Topological Alignment Loss).}
\label{tab:notation_2_1}
\begin{tabular}{@{}ll@{}}
\toprule
\textbf{Symbol} & \textbf{Description} \\
\midrule
$X$ & English captions (source-language captions). \\
$X^{*}$ & Translated captions in a target language. \\[2pt]

$E_T$ & CLIP text encoder (teacher), encoding $X$. \\
$E_S$ & Multilingual CLIP (MCLIP) text encoder (student), encoding $X^{*}$. \\[2pt]

$\{(I_i, T_i)\}_{i=1}^{N}$ & A minibatch of $N$ image-caption pairs. \\
$T_i^{*}$ & Translated caption of $T_i$. \\[2pt]

$\{E_T(T_i)\}_{i=1}^{N}$ & Teacher text-embedding point cloud from the batch $\{T_i\}_{i=1}^N$. \\
$\{E_S(T_i^{*})\}_{i=1}^{N}$ & Student text-embedding point cloud from the batch $\{T_i^{*}\}_{i=1}^N$. \\[2pt]

$D_T$ & Persistence diagram computed from point cloud $\{E_T(T_i)\}_{i=1}^{N}$. \\
$D_S$ & Persistence diagram computed from point cloud $\{E_S(T_i^{*})\}_{i=1}^{N}$. \\[2pt]

$L_{\mathrm{pw}}$ & Point-wise alignment loss proposed by MCILP. \\
$L_{\mathrm{ta}}$ & Topological alignment loss we proposed. \\
$L_{\mathrm{dm}}$ & Distance matrix loss we proposed. \\[2pt]

$SW^{(K)}_{p}(\cdot,\cdot)$ & Sliced $p$-Wasserstein distance with $K$ projection directions. \\[2pt]

$M_X$ & Pairwise distance matrix of $X=\{x_i\}_{i=1}^{N}$:
$(M_X)_{i,j}=\|x_i-x_j\|_2$. \\
$M_T$ & Pairwise distance matrix computed from $T=\{E_T(T_i)\}_{i=1}^{N}$. \\
$M_S$ & Pairwise distance matrix computed from $S=\{E_S(T_i^{*})\}_{i=1}^{N}$. \\[2pt]

$L_{\mathrm{total}}$ & Total objective: $L_{\mathrm{total}}=\alpha L_{\mathrm{pw}}+\beta L_{\mathrm{ta}}+\gamma L_{\mathrm{dm}}$. \\
$\alpha,\beta,\gamma$ & Loss weights (hyperparameters). \\[2pt]

$D^{(k)}_X$ & $k$-dimensional persistence diagram of point cloud $X$ (stability discussion). \\
$W_p(\cdot,\cdot)$ & $p$-Wasserstein distance between persistence diagrams. \\
$W^{c}_{p}(X,Y)$ & $p$-Wasserstein distance between point clouds $X$ and $Y$ (stability bound). \\
$C_k$ & Stability constant in $W_p(D^{(k)}_X, D^{(k)}_Y)\le C_k W^{c}_{p}(X,Y)$. \\
\bottomrule
\end{tabular}
\end{table}

Table~\ref{tab:notation_2_1} summarizes the notations used in Section~\ref{sec:topological_alignment}.

\section{PROOF OF THEOREM}
\label{app:proof}

\setcounter{theorem}{0}
\begin{theorem}
    Let $0 \le \epsilon \le 1$. Define $G_\epsilon = (V,E,\omega_\epsilon)$ by
    \begin{equation}
        \omega_\epsilon(e) =
        \begin{cases}
            \omega(e), & \text{if } \omega(e) \le \epsilon,\\[4pt]
            1,         & \text{if } \omega(e) > \epsilon.
        \end{cases}
    \end{equation}
    Let
        $m(\epsilon) \coloneqq 
        \#\bigl\{ (0,d) \in D_0^{\mathrm{Rips}}(G) \mid \epsilon < d < \infty \bigr\}$,
    i.e., the number of finite $0$-dimensional persistence points of $G$ 
    whose death times exceed $\epsilon$.  
    Then
    \begin{equation}
        W_p\!\bigl(D_0^{\mathrm{Rips}}(G),\,D_0^{\mathrm{Rips}}(G_\epsilon)\bigr)
        \;\le\; m(\epsilon)^{1/p}\,(1-\epsilon)
    \end{equation}
    and $0 \leq m(\epsilon) \leq N-1$ where $W_p$ denotes $p$-Wasserstein distance.
\end{theorem}

\begin{proof}
Let $\mathcal{F}_G=\{\mathrm{VR}_\alpha(G)\}_{\alpha\ge0}$ and
$\mathcal{F}_{G_\epsilon}=\{\mathrm{VR}_\alpha(G_\epsilon)\}_{\alpha\ge0}$ denote the (graph-level) $1$-skeleton filtrations where
\[
  \mathrm{VR}_\alpha(G)= V \cup \{\,e\in E \mid \omega(e)\le \alpha\,\},\qquad
  \mathrm{VR}_\alpha(G_\epsilon)= V \cup \{\,e\in E \mid \omega_\epsilon(e)\le \alpha\,\}.
\]
Since $0$-dimensional homology is depends only on $0$ and $1$-simplices, it suffices to consider the filtered $1$-skeleton.
For $\alpha\le \epsilon$, we have $\omega_\epsilon(e)=\omega(e)$ whenever $\omega(e)\le \epsilon$, hence
$\mathrm{VR}_\alpha(G)=\mathrm{VR}_\alpha(G_\epsilon)$.
Moreover, since $\omega_\epsilon(e)\in\{\omega(e),1\}$, for every $\alpha$ with $\epsilon<\alpha<1$ 
we have $\mathrm{VR}_\alpha(G_\epsilon)=\mathrm{VR}_\epsilon(G_\epsilon)$, i.e., the filtration of $G_\epsilon$ is constant on $[\epsilon,1)$.
Consequently, in $D_0^{\mathrm{Rips}}(G_\epsilon)$ every class that is still alive at time $\epsilon$ dies precisely at $\alpha=1$ when all remaining edges of weight $1$ are added.

In $0$-dimensional persistence points, all births occur at $0$, and there are $N$ points including a single essential class. 
Thus, points of $D_0^{\mathrm{Rips}}(G)$ with death times $d\le \epsilon$ also appear with the same deaths in $D_0^{\mathrm{Rips}}(G_\epsilon)$, while each point with death $d\in(\epsilon,1)$ in $D_0^{\mathrm{Rips}}(G)$ corresponds to a point with death $1$ in $D_0^{\mathrm{Rips}}(G_\epsilon)$.

Define a bijection $\gamma': D_0^{\mathrm{Rips}}(G)\cup \Delta \to D_0^{\mathrm{Rips}}(G_\epsilon)\cup \Delta$ by
\begin{equation}
  \gamma'(0,d) \;=\;
  \begin{cases}
    (0,d), & d\le \epsilon,\\
    (0,1), & \epsilon< d \le 1,
  \end{cases}
\end{equation}
and map the essential class to the essential class. (No diagonal points are used here, but allowing $\Delta$ keeps the statement standard.)
With the usual $\ell_p$ ground metric on $\mathbb{R}^2$, we have
\begin{equation}
  \| (0,d) - \gamma'(0,d) \|_p \;=\;
  \begin{cases}
    0, & d\le \epsilon,\\
    |1-d|, & \epsilon<d\le 1.
  \end{cases}
\end{equation}
The number $m(\epsilon)$ of pairs with $\epsilon<d\le 1$ is at most $N-1$ (all but the essential component). Therefore,
\begin{equation}
  \sum_{u \in D_0^{\mathrm{Rips}}(G)\cup \Delta} (\|u-\gamma'(u)\|_p)^{\,p}
  \; < \; m(\epsilon)\,(1-\epsilon)^{p},
\end{equation}
and $0 \leq m(\epsilon) \leq N-1$ since $|1-d| < 1-\epsilon$ for every $d\in(\epsilon,1]$.
Taking the infimum over all bijections and the $p$-th root yields
\begin{align}
  W_p\!\bigl(D_0^{\mathrm{Rips}}(G),\,D_0^{\mathrm{Rips}}(G_\epsilon)\bigr)
  \; 
  &< \; \bigl( m(\epsilon)\,(1-\epsilon)^{p} \bigr)^{1/p}
  \; \\
  &=\; m(\epsilon)^{1/p}\,(1-\epsilon),
\end{align}
which proves the claim.
\end{proof}

\section{DATASETS AND EXPERIMENTAL DETAILS}
\label{app:data}

\paragraph{Datasets.}
We use the multilingual caption dataset introduced by~\citep{carlsson2022cross}, publicly available at \url{https://huggingface.co/datasets/M-CLIP/ImageCaptions-7M-Translations}. 
While the corpus provides translations for multiple languages, Korean is not included. 
To incorporate Korean, we augment the corpus by replacing a portion of captions with Korean translations; the replacement ratio and exact sampling procedure are specified below.

\begin{table*}[htbp]
\centering
\begingroup
\setlength{\tabcolsep}{5pt}
\renewcommand{\arraystretch}{1.05}
\caption{Per-language sample counts before/after adding Korean. Before: all languages except Vietnamese had 150{,}000; Vietnamese had 100{,}000; Korean was absent. Totals are preserved.}
\begin{tabular}{lrrr@{\hspace{1.8em}}lrrr}
\toprule
Language & Before & After & $\Delta$ & Language & Before & After & $\Delta$ \\
\midrule
afrikaans           & 150000 & 147000 & $-3000$ & italian      & 150000 & 147000 & $-3000$ \\
albanian            & 150000 & 147000 & $-3000$ & japanese     & 150000 & 147000 & $-3000$ \\
amharic             & 150000 & 147000 & $-3000$ & korean       &      0 & 138000 & $+138000$ \\
arabic              & 150000 & 147000 & $-3000$ & macedonian   & 150000 & 147000 & $-3000$ \\
azerbaijani         & 150000 & 147000 & $-3000$ & malayalam    & 150000 & 147000 & $-3000$ \\
bengali             & 150000 & 147000 & $-3000$ & marathi      & 150000 & 147000 & $-3000$ \\
bosnian             & 150000 & 147000 & $-3000$ & polish       & 150000 & 147000 & $-3000$ \\
bulgarian           & 150000 & 147000 & $-3000$ & portuguese   & 150000 & 147000 & $-3000$ \\
catalan             & 150000 & 147000 & $-3000$ & romanian     & 150000 & 147000 & $-3000$ \\
chinese\_simplified & 150000 & 147000 & $-3000$ & russian      & 150000 & 147000 & $-3000$ \\
chinese\_traditional& 150000 & 147000 & $-3000$ & serbian      & 150000 & 147000 & $-3000$ \\
czech               & 150000 & 147000 & $-3000$ & slovenian    & 150000 & 147000 & $-3000$ \\
danish              & 150000 & 147000 & $-3000$ & spanish      & 150000 & 147000 & $-3000$ \\
dutch               & 150000 & 147000 & $-3000$ & swahili      & 150000 & 147000 & $-3000$ \\
english             & 150000 & 147000 & $-3000$ & swedish      & 150000 & 147000 & $-3000$ \\
estonian            & 150000 & 147000 & $-3000$ & tagalog      & 150000 & 147000 & $-3000$ \\
french              & 150000 & 147000 & $-3000$ & telugu       & 150000 & 147000 & $-3000$ \\
german              & 150000 & 147000 & $-3000$ & turkish      & 150000 & 147000 & $-3000$ \\
greek               & 150000 & 147000 & $-3000$ & turkmen      & 150000 & 147000 & $-3000$ \\
hindi               & 150000 & 147000 & $-3000$ & ukrainian    & 150000 & 147000 & $-3000$ \\
hungarian           & 150000 & 147000 & $-3000$ & uzbek        & 150000 & 147000 & $-3000$ \\
icelandic           & 150000 & 147000 & $-3000$ & uyghur       & 150000 & 147000 & $-3000$ \\
indonesian          & 150000 & 147000 & $-3000$ & vietnamese   & 100000 & 100000 & $0$     \\
\midrule
\multicolumn{8}{r}{\textbf{Total}\quad Before: \textbf{7000000}\quad After: \textbf{7000000}\quad $\Delta$: \textbf{0}}\\
\bottomrule
\end{tabular}
\label{tab:lang_counts_before_after}
\endgroup
\end{table*}

\paragraph{Korean Augmentation.}

In the original corpus, Korean was absent; 46 languages had 150{,}000 captions each and Vietnamese had 100{,}000, totaling 7M samples. 
We added Korean while preserving the per-language ratios and the total size by uniformly reallocating 3{,}000 captions from each non-Vietnamese language to Korean. 
Specifically, for every language except Vietnamese (fixed at 100{,}000), we randomly selected 3{,}000 captions and replaced them with Korean translations. 
This results in 147{,}000 samples per non-Vietnamese language (down from 150{,}000) and 138{,}000 Korean samples in total ($46\times3{,}000$). 
Table~\ref{tab:lang_counts_before_after} summarizes the per-language counts. 

Korean translations were generated using the OpenAI API with a temperature setting of 0.0 to ensure deterministic and consistent translations. To handle the large-scale translation task efficiently, we implemented a batch processing pipeline with checkpoint mechanisms. The translation system processed captions in batches of 1,000 items, with automatic checkpointing every 5,000 translations to enable recovery from potential interruptions. Each translation request included explicit instructions to return only the translated text without additional formatting or explanations. Failed translation attempts were handled with exponential backoff retry logic (up to 3 attempts) to ensure robustness against transient API failures.

\paragraph{Embedding Subset.}
Although the full dataset contains approximately 7M samples, we rely on the 2M precomputed text embeddings released at \href{https://huggingface.co/datasets/M-CLIP/ImageCaptions-7M-Embeddings}{\texttt{ImageCaptions-7M-Embeddings}}. 
We use this subset to train both MCLIP and ToMCLIP and verify that it is sufficient to reproduce the MCLIP performance reported in~\citep{carlsson2022cross}.
To evaluate the model under a low-resource condition, we further subsampled 1\% of the 2M samples and trained MCLIP and ToMCLIP using this reduced training set. 
This setup simulates scenarios where access to multilingual annotated data is severely limited.

\paragraph{Models.}
For multilingual text encoding, we adopt XLM-RoBERTa~\citep{conneau2019unsupervised}.
We use the CLIP (ViT-B/32) image encoder~\citep{radford2021learning}.
When comparing MCLIP and ToMCLIP, the backbone architecture, optimizer, and learning-rate schedule are identical unless otherwise noted.
We set the batch size to $256$, following MCLIP~\citep{carlsson2022cross}. 
ToMCLIP($L_{\text{dm}}$), ToMCLIP($L_{\text{ta}}$), and ToMCLIP denote models trained with the proposed total loss \(L_{\text{total}}\) using coefficients \((\alpha,\beta,\gamma)=(1,0.01,0)\), \((1,0,0.01)\), and \((1,0.01,0.01)\), respectively.
To construct a sparse graph, let $DM$ denote the pairwise distance matrix; we set $\epsilon = \text{mean}(DM) - 0.5*\text{std}(DM)$, computed separately for each point cloud.
For the sliced Wasserstein distance, we use $p=2$  and average over $50$ random projection directions.

\paragraph{Training and Evaluation.}
We train under two data regimes: full-data (all available subset entries) and a 1\% low-resource setting.
We report zero-shot CIFAR-100 classification across 13 languages using top-1/5/10.
All preprocessing, tokenization settings, batch sizes, learning rates, and early stopping are the same as MCLIP~\citep{carlsson2022cross}, except for the loss function, which includes our topology-alignment objective.

\section{PROMPTS OF MULTILINGUAL LANGUAGE FOR THE EVALUATION OF ZERO-SHOT CLASSIFICATION ON THE CIFAR-100}
\label{app:prompts}

To perform zero-shot classification on the CIFAR-100 dataset, we construct language-specific text prompts to match the expected format of each language. These prompts are used to generate class-specific textual descriptions, which are then embedded using the multilingual text encoder. The general template follows the format of ``a photo of a \{\}'' in English, where the placeholder is replaced by the class name.
Table~\ref{tab:prompt_templates} summarizes the prompt templates used for each language in our evaluation.

\begin{table}[!ht]
\centering
\caption{Prompt templates used for each language in the zero-shot classification task. The placeholder \{\} is replaced with the class name.}
\label{tab:prompt_templates}
\begin{tabular}{ll}
\toprule
\textbf{Language (ISO)} & \textbf{Prompt Template} \\
\midrule
English (En)    & \texttt{a photo of a \{\}} \\
French (Fr)     & \texttt{une photo d\'un(e) \{\}} \\
Spanish (Es)    & \texttt{una foto de un(a) \{\}} \\
German (De)     & \texttt{ein Foto von einem/einer \{\}} \\
Italian (It)    & \texttt{una foto di un(a) \{\}} \\
Russian (Ru)    & \ru{фото \{\}} \\
Polish (Pl)     & \pltt{zdjęcie \{\}} \\
Turkish (Tr)    & \texttt{\{\} foto\u{g}raf\i} \\
Danish (Da)     & \texttt{et billede af en \{\}} \\
Japanese (Ja)   & \jatt{\{\}の写真} \\
Chinese (Zh)    & \zhtt{一张\{\}的照片} \\
Korean (Ko)     & \kott{\{\}가\;있는\;사진} \\
Vietnamese (Vi) & \vi{một bức ảnh về \{\}} \\
\bottomrule
\end{tabular}
\end{table}

\section{ADDITIONAL RESULTS}
\label{app:results}

\subsection{Evaluation on CIFAR-100}
\label{app:top-15_results}

\begin{table*}[!ht]
\centering
\caption{Top-1 accuracy (\%) of zero-shot classification on CIFAR-100 across 13 languages (Full vs. Low).}
\label{tab:cifar100_lang_top1}
\resizebox{\textwidth}{!}{%
\begin{tabular}{llcccccccccccccc}
\toprule
& & \multicolumn{13}{c}{Languages (13)} & \multirow{2}{*}{\textit{Avg}} \\
\cmidrule(lr){3-15}
Setting & Model & En & Fr & Es & De & It & Ru & Pl & Tr & Da & Ja & Zh & Ko & Vi \\
\midrule
 & CLIP & \textbf{60.67} & 40.11 & 37.49 & 36.06 & 26.93 & 1.06 & 10.71 & 9.54 & 17.87 & 12.40 & 5.21 & 2.21 & 3.49 & 20.29 \\
 & MCLIP & 58.86 & 49.14 & 51.13 & 51.23 & 51.13 & 49.83 & \textbf{51.40} & \textbf{51.24} & 55.13 & 33.01 & \textbf{54.70} & 51.16 & 51.35 & 50.72 \\
\textbf{Full data} & ToMCLIP($L_{\text{dm}}$) & 57.79 & 46.19 & 50.39 & \textbf{56.13} & 50.39 & 48.62 & 50.29 & 50.99 & 56.62 & \textbf{33.85} & 52.35 & 52.28 & 51.03 & 50.53 \\
\textbf{(2M smaples)} & ToMCLIP($L_{\text{ta}}$) & 58.10 & 48.67 & 48.54 & 52.42 & 51.44 & \textbf{52.67} & 50.74 & 50.57 & 57.09 & 32.86 & 51.90 & 51.37 & \textbf{53.15} & 50.73 \\
 & ToMCLIP & 58.93 & \textbf{50.76} & \textbf{52.67} & 54.27 & \textbf{52.68} & 50.63 & 50.04 & 51.21 & \textbf{57.50} & 31.33 & 52.97 & \textbf{52.41} & 51.72 & \textbf{51.32} \\
\midrule
\multirow{5}{*}{\textbf{Low}} & CLIP & \textbf{60.67} & \textbf{40.11} & \textbf{37.49} & \textbf{36.06} & 26.93 & 1.06 & 10.71 & 9.54 & 17.87 & 12.40 & 5.21 & 2.21 & 3.49 & 20.29 \\
& MCLIP & 35.70 & 32.40 & 29.64 & 31.20 & 28.19 & 32.21 & 27.25 & 25.05 & 33.88 & 24.41 & 33.63 & 30.38 & 28.77 & 30.21 \\
\textbf{Low resource} & ToMCLIP($L_{\text{dm}}$) & 37.84 & 33.12 & 30.32 & 31.13 & 29.82 & 32.70 & 28.87 & 25.16 & 35.24 & 25.91 & 34.32 & 31.27 & 28.82 & 31.12 \\
\textbf{(1\% subset)} & ToMCLIP($L_{\text{ta}}$) & 37.79 & 31.01 & 29.75 & 31.25 & 28.82 & 32.07 & 28.18 & 24.43 & 34.49 & 23.87 & 32.79 & 30.75 & \textbf{30.67} & 30.45 \\
& ToMCLIP & 37.64 & 34.08 & 31.12 & 31.09 & \textbf{31.28} & \textbf{34.08} & \textbf{30.20} & \textbf{25.75} & \textbf{36.11} & \textbf{26.65} & \textbf{35.18} & \textbf{31.79} & 29.90 & \textbf{31.91} \\
\bottomrule
\end{tabular}
}
\end{table*}
\begin{table*}[!ht]
\centering
\caption{Top-5 accuracy (\%) of zero-shot classification on CIFAR-100 across 13 languages (Full vs. Low).}
\label{tab:cifar100_lang_top5}
\resizebox{\textwidth}{!}{%
\begin{tabular}{llcccccccccccccc}
\toprule
& & \multicolumn{13}{c}{Languages (13)} & \multirow{2}{*}{\textit{Avg}} \\
\cmidrule(lr){3-15}
Setting & Model & En & Fr & Es & De & It & Ru & Pl & Tr & Da & Ja & Zh & Ko & Vi \\
\midrule
\multirow{5}{*}{\textbf{Full}} & CLIP & 85.26 & 58.75 & 56.94 & 55.17 & 42.02 & 6.49 & 16.71 & 17.56 & 27.47 & 25.33 & 14.26 & 6.74 & 9.44 & 32.47 \\
 & MCLIP & \textbf{85.38} & 77.07 & 78.25 & 77.13 & 79.41 & 79.06 & \textbf{76.51} & 78.06 & 79.98 & 46.85 & \textbf{81.39} & 77.86 & 77.45 & 76.49 \\
\textbf{Full data} & ToMCLIP($L_{\text{dm}}$) & 84.23 & 73.35 & 73.30 & \textbf{82.06} & 77.03 & 76.31 & 74.19 & 78.61 & 79.84 & 49.05 & 79.75 & 79.40 & 78.85 & 75.84 \\
\textbf{(2M smaples)} & ToMCLIP($L_{\text{ta}}$) & 84.22 & 75.25 & 74.00 & 79.58 & 79.96 & \textbf{80.76} & 76.09 & \textbf{79.58} & 80.80 & \textbf{50.10} & 81.12 & 79.28 & \textbf{81.83} & 77.12 \\
 & ToMCLIP & 84.78 & \textbf{78.87} & \textbf{79.11} & 80.97 & \textbf{80.09} & 78.39 & 74.66 & 78.89 & \textbf{81.27} & 49.58 & 80.38 & \textbf{79.79} & 80.16 & \textbf{77.46} \\
\midrule
\multirow{5}{*}{\textbf{Low}} & CLIP & \textbf{85.26} & \textbf{58.75} & \textbf{56.94} & 55.17 & 42.02 & 6.49 & 16.71 & 17.56 & 27.47 & 25.33 & 14.26 & 6.74 & 9.44 & 32.47 \\
& MCLIP & 67.99 & 57.26 & 52.52 & \textbf{60.26} & 50.52 & 57.82 & 52.05 & 48.35 & 61.92 & 49.45 & 67.07 & 54.48 & 57.00 & 56.67 \\
\textbf{Low resource} & ToMCLIP($L_{\text{dm}}$) & 67.70 & 58.15 & 53.18 & 59.20 & 51.96 & 56.97 & 51.52 & 46.31 & 62.34 & 48.56 & 65.88 & 56.48 & 55.85 & 56.47 \\
\textbf{(1\% subset)} & ToMCLIP($L_{\text{ta}}$) & 68.39 & 57.19 & 52.97 & 59.18 & 51.44 & 58.32 & 53.15 & 48.51 & 61.70 & 50.08 & 65.50 & 56.15 & \textbf{60.28} & 57.14 \\
& ToMCLIP & 68.75 & 58.42 & 54.09 & 60.12 & \textbf{53.73} & \textbf{59.50} & \textbf{54.07} & \textbf{49.92} & \textbf{63.29} & \textbf{51.29} & \textbf{67.36} & \textbf{56.62} & 58.74 & \textbf{58.15} \\
\bottomrule
\end{tabular}
}
\end{table*}

In this section, we report Top-1 and Top-5 performance on CIFAR-100 under both the full-resource and low-resource settings, where the results for the low-resource setting are averaged over three independent runs. 
As shown in Tables~\ref{tab:cifar100_lang_top1} and~\ref{tab:cifar100_lang_top5}, ToMCLIP outperforms MCLIP in zero-shot classification on CIFAR-100 across 13 languages. 
These results confirm that topology-aware alignment enhances cross-lingual consistency and robustness.

\subsection{Ablation Study on Batch Size}
\label{app:batch_size}

\begin{figure*}[!ht]
  \centering
  \includegraphics[width=0.65\linewidth]{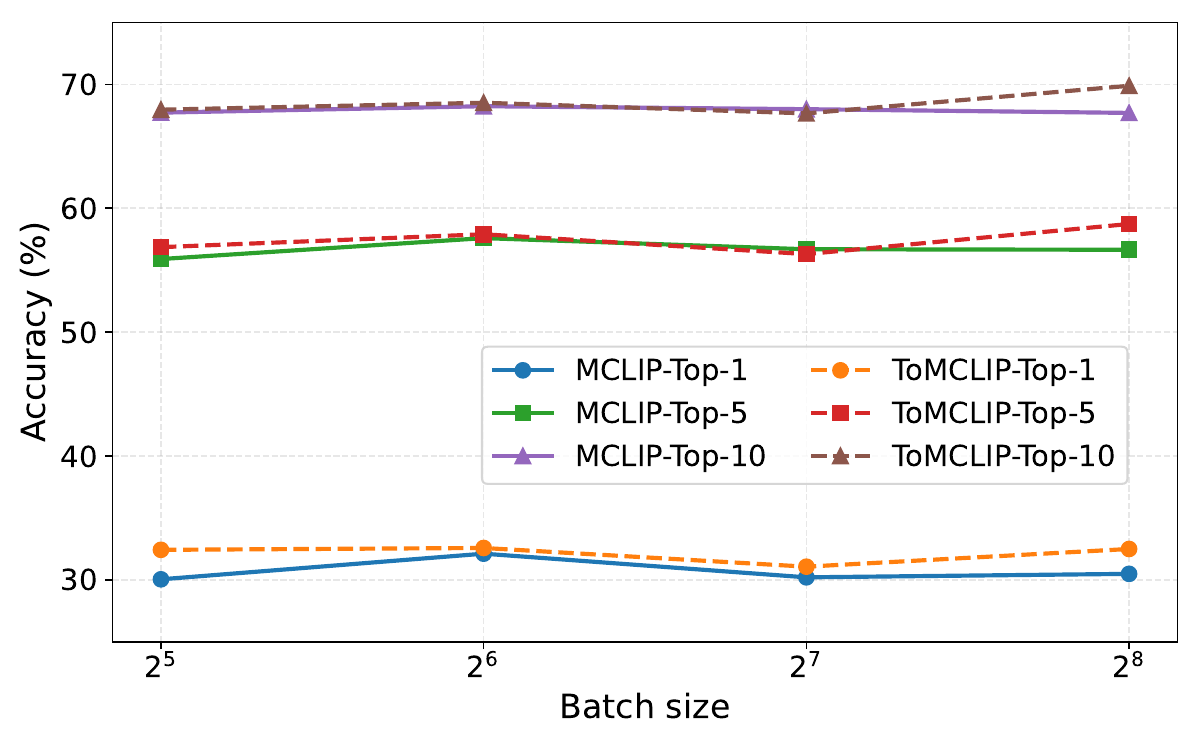} 
  \caption{Ablation study on batch size in the low-resource setting.}
  \label{fig:batchsize_ablation}
\end{figure*}

We also investigate the effect of batch size in the low-resource setting. 
In our framework, the batch size corresponds to the number of sampled points considered when constructing persistence diagrams in the shared embedding space. 
Hence, a larger batch size allows for capturing more refined topological features and yields better approximations of the underlying geometry. 
However, increasing the batch size raises computational complexity, making it crucial to balance accuracy and efficiency. 
As shown in Figure~\ref{fig:batchsize_ablation}, performance improves with larger batches, 
and we therefore adopt a batch size of $256$ as the default in the experiments. 
We note that with batch sizes smaller than $128$, 
the number of sampled points is insufficient to approximate the underlying data manifold, 
leading to limited improvements. 
By contrast, a batch size of $256$ provides enough samples to extract topological information more effectively. 
Further exploration with $512$ or larger batch sizes may reveal 
whether additional gains are possible, which we leave for future work. 
In addition, future work will also explore approximation techniques to further reduce computational cost while maintaining the benefits of large batch sizes.

\subsection{Ablation Study on Loss Coefficients}
\label{app:loss_coeffi}

In Table~\ref{tab:loss_coeffi}, we present the ablation study on the loss coefficients under the low-resource setting. 
We observe that extremely large coefficients (e.g., $\beta=0.1$ or $\gamma=0.1$) severely degrade performance across all metrics, 
while small to moderate values (e.g., $\beta=0.01, 0.001$ or $\gamma=0.01, 0.001$) provide stable performance over the baseline. 
Among the tested configurations, $\beta=0.01$ and $\gamma=0.01$ achieve the highest scores for both Top-1 (32.49\%) and Top-5 (58.73\%), 
as well as the best Top-10 accuracy (69.89\%). 
Therefore, we adopt $\beta=0.01$ and $\gamma=0.01$ as the default setting for the all experiments.

\begin{table}[htbp]
\centering
\caption{Ablation results on loss coefficients. The experiments are conducted on \textbf{Low} resource setting.}
\label{tab:loss_coeffi}
\resizebox{0.8\textwidth}{!}{%
\begin{tabular}{lcccccccccccc}
\toprule
 & \multicolumn{4}{c}{Top-1 (\%)} & \multicolumn{4}{c}{Top-5 (\%)} & \multicolumn{4}{c}{Top-10 (\%)} \\
\cmidrule(lr){2-5} \cmidrule(lr){6-9} \cmidrule(lr){10-13}
\diagbox{$\beta$}{$\gamma$} & 0.0 & 0.001 & 0.01 & 0.1 & 0.0 & 0.001 & 0.01 & 0.1 & 0.0 & 0.001 & 0.01 & 0.1 \\
\midrule
0.0 & 30.48 & 29.89 & 31.20 & 1.00 & 56.65 & 55.86 & 56.43 & 5.07 & 67.70 & 67.54 & 67.65 & 9.97 \\
0.001 & 30.38 & 29.40 & 30.74 & 1.10 & 56.26 & 55.57 & 56.23 & 4.80 & 67.54 & 66.80 & 67.42 & 9.69 \\
0.01 & 30.16 & 29.67 & \textbf{32.49} & 1.01 & 56.87 & 55.99 & \textbf{58.73} & 5.28 & 68.53 & 67.44 & \textbf{69.89} & 10.23 \\
0.1 & 1.14 & 1.19 & 1.31 & 1.00 & 5.21 & 5.01 & 6.60 & 5.11 & 10.49 & 10.15 & 13.03 & 10.10 \\
\bottomrule
\end{tabular}
}%
\end{table}

\subsection{Ablation on 1-Dimensional Homology}
\label{app:onedim}

Our model uses only $0$-dimensional homology ($H_0$) to extract topological features, since the birth times of $1$-dimensional homology ($H_1$) features largely overlap with the pairwise distance MSE ($L_{\mathrm{dm}}$).
To verify this, we empirically tested whether incorporating $H_1$ improves training.
Specifically, we defined
\begin{equation}
    L_{\mathrm{ta}}
    =\frac{1}{2}\,\mathrm{SW}_p^{(K)}\!\big(D_T^{(0)}, D_S^{(0)}\big)
    +\frac{1}{2}\,\mathrm{SW}_p^{(K)}\!\big(D_T^{(1)}, D_S^{(1)}\big),
\end{equation}
and set $(\alpha,\beta,\gamma)=(1,\,0.01,\,0.01)$.
We conducted experiments on \textbf{Low} setting.
Adding $H_1$ lowers the overall average (69.89 $\rightarrow$ 69.03), suggesting that $H_1$ provides limited additional benefit in our setting.

\begin{table*}[!ht]
\centering
\caption{Ablation on 1-dimensional homology. ToMCLIP(dim1) denotes the model trained with $L_{\mathrm{ta}}$ computed on both H$_0$ and H$_1$.}

\label{tab:onedim}
\resizebox{\textwidth}{!}{%
\begin{tabular}{lcccccccccccccc}
\toprule
& \multicolumn{13}{c}{Languages (13)} & \multirow{2}{*}{\textit{Avg}} \\
\cmidrule(lr){2-14}
Model & En & Fr & Es & De & It & Ru & Pl & Tr & Da & Ja & Zh & Ko & Vi \\
\midrule
MCLIP & 79.25 & 67.60 & 62.21 & 70.44 & 60.32 & 69.41 & 64.64 & 57.87 & 72.95 & 62.09 & 77.32 & 64.72 & 71.24 & 67.70 \\
ToMCLIP & 81.06 & \textbf{70.66} & \textbf{64.25} & 72.70 & \textbf{63.54} & \textbf{71.88} & \textbf{67.04} & \textbf{60.87} & \textbf{74.77} & \textbf{64.21} & \textbf{78.33} & 67.23 & 71.99 & \textbf{69.89} \\
ToMCLIP(dim1) &  \textbf{82.27} & 69.73 & 64.02 &  \textbf{73.72} & 61.16 & 71.87 & 64.53 & 60.11 & 73.58 & 60.08 & 78.07 & 66.44 & 71.81 & 69.03 \\
\bottomrule
\end{tabular}
}
\end{table*}

\subsection{Effect of the Approximation Threshold for the Persistence Diagram}
\label{app:std_scale}
We control graph sparsity with a distance threshold
$\epsilon = \mu - \lambda\sigma$, where $\mu$ and $\sigma$ denote the mean and standard deviation of pairwise distances, respectively; we keep edges with
distance $\le \epsilon$.
We conducted an ablation study on $\lambda \in \{1.5, 1, 0.5, 0\}$ in the low-resource setting.
As $\lambda$ increases, $\epsilon$ decreases and the graph becomes sparser, which
reduces memory/time but may remove informative structure.
Table~\ref{tab:std_scale} summarizes the results. Increasing $\lambda$ makes the
graph sparser and speeds up persistence diagram computation (0.075\,s
$\rightarrow$ 0.011\,s from $\lambda{=}0$ to $1.5$), but excessive sparsity hurts accuracy. 
As the graph becomes denser (smaller $\lambda$), the persistence diagram approximation approaches the exact persistence diagram and accuracy does not decrease. 
In practice, $\lambda=0.5$ already makes the approximation error negligible.
Choosing $\lambda<0.5$ increases computation without yielding further gains, whereas $\lambda>0.5$ introduces additional sparsity, incurs approximation error, and lowers accuracy. 
Consistent with the analysis in Section~\ref{sec:conn_spars}, the persistence diagram approximation error near $\lambda=0.5$ is negligible, which supports adopting \textbf{$\lambda=0.5$} as the default balance between performance and computational cost.

\begin{table}[!ht]
\caption{Top-10 accuracy (\%) of zero-shot classification on CIFAR-100 across 13 languages.}
\centering
\resizebox{\textwidth}{!}{%
\begin{tabular}{l|rrrrrrrrrrrrrrr}
\toprule
$\lambda$ & Time(s) & En & Fr & Es & De & It & Ru & Pl & Tr & Da & Ja & Zh & Ko & Vi & Avg \\
\midrule
1.5 & \textbf{0.01119} & 78.73 & 66.95 & 62.78 & 69.67 & 60.10 & 68.71 & 62.82 & 57.38 & 70.80 & 57.57 & 75.08 & 65.76 & 71.67 & 66.77 \\
1 & 0.02434 & 80.50 & 68.89 & 63.22 & 71.75 & 61.96 & 69.80 & 63.59 & 60.82 & 73.06 & 59.97 & 77.87 & 68.12 & 71.82 & 68.57 \\
0.5 & 0.04608 & \textbf{81.06} & \textbf{70.66} & \textbf{64.25} & \textbf{72.70} & \textbf{63.54} & \textbf{71.88} & \textbf{67.04} & \textbf{60.87} & \textbf{74.77} & \textbf{64.21} & \textbf{78.33} & 67.23 & \textbf{71.99} & \textbf{69.89} \\
0 & 0.07519 & 80.42 & 70.13 & 63.32 & 70.21 & 61.91 & 71.67 & 65.51 & 59.19 & 73.78 & 61.91 & 78.21 & \textbf{68.20} & 71.84 & 68.95 \\
\bottomrule
\end{tabular}}
\label{tab:std_scale}
\end{table}

\subsection{Effect of the Number of Projections for SWD}
\label{app:swd_K}

We approximate the SWD in Eq.~\ref{eq:swd} via Monte Carlo sampling with $K$ random projection directions.
The computational complexity is $\mathcal{O}(K\,N\log N)$; since runtime grows approximately linearly with $K$, we ablate $K$ to select a balanced value.
Table~\ref{tab:swdK_cifar100_k10} reports the ablation in the low-resource setting.
Empirically, increasing $K$ improves accuracy up to a point: the average Top-10 accuracy rises from 66.99 ($K{=}5$) and 66.81 ($K{=}10$) to 68.16 ($K{=}30$), peaking at 69.89 ($K{=}50$).
For $K{=}50$ and $K{=}100$, performance is similar while the computational cost roughly doubles; hence we adopt \textbf{$K{=}50$} as the default.

\begin{table}[!ht]
\caption{Top-10 accuracy (\%) of zero-shot classification on CIFAR-100 across 13 languages.}
\centering
\resizebox{\textwidth}{!}{%
\begin{tabular}{l|rrrrrrrrrrrrrr}
\toprule
$K$ & En & Fr & Es & De & It & Ru & Pl & Tr & Da & Ja & Zh & Ko & Vi & Avg \\
\midrule
5 & 78.04 & 66.92 & 60.42 & 69.17 & 59.95 & 71.64 & 63.43 & 55.38 & 72.69 & 58.69 & 77.32 & 65.25 & 71.93 & 66.99 \\
10 & 78.03 & 66.08 & 61.88 & 70.07 & 60.06 & 69.05 & 63.45 & 57.91 & 71.41 & 58.00 & 76.50 & 64.61 & 71.51 & 66.81 \\
30 & 79.47 & 67.18 & 63.51 & 70.57 & 60.48 & 69.53 & 66.41 & 58.76 & 72.10 & 59.19 & 77.84 & \textbf{67.92} & 73.13 & 68.16 \\
50 & \textbf{81.06} & 70.66 & 64.25 & \textbf{72.70} & 63.54 & \textbf{71.88} & \textbf{67.04} & \textbf{60.87} & \textbf{74.77} & \textbf{64.21} & \textbf{78.33} & 67.23 & 71.99 & \textbf{69.89} \\
100 & 79.58 & \textbf{71.81} & \textbf{66.18} & 71.97 & \textbf{64.61} & 70.92 & 64.30 & 57.44 & 73.07 & 58.33 & 78.17 & 66.63 & \textbf{73.30} & 68.95 \\
\bottomrule
\end{tabular}}
\label{tab:swdK_cifar100_k10}
\end{table}

\subsection{Evaluation on xFlickr\&CO}
\label{app:retrieval_xflickrco}

\begin{table*}[!ht]
\centering
\caption{Multilingual retrieval on xFlickr\&CO. R@1 retrieval accuracy (\%) across languages.
\textcolor{red}{$\blacktriangle$} and \textcolor{blue}{$\blacktriangledown$} mark improvements/decreases over MCLIP
for the same setting and direction; here only the icons are shown.}
\label{tab:xflickrco_top1}
\resizebox{0.9\textwidth}{!}{%
\small
\begin{tabular}{lllccccccccc}
\toprule
\multirow{3}{*}{Setting} & \multirow{3}{*}{Direction} &  \multirow{3}{*}{Model} & \multicolumn{8}{c}{Languages} & \multirow{2}{*}{\textit{Avg}} \\
\cmidrule(lr){4-11}
 & &  & En & Es & De & Id & Ru & Tr & Ja & Zh &  \\
\midrule
 & \multirow{5}{*}{\textbf{IR}} & CLIP & 54.90 & 22.05 & 11.00 & 4.15 & 0.35 & 1.90 & 1.95 & 0.35 & 12.08 \\
 & & MCLIP & 55.00 & 54.65 & 48.45 & 48.95 & 56.65 & 53.35 & 35.45 & 48.50 & 50.12 \\
 & & ToMCLIP($L_{\text{dm}}$) & 55.10\,\uptriangle & 55.10\,\uptriangle & 48.65\,\uptriangle & 49.50\,\uptriangle & 56.95\,\uptriangle & \textbf{54.35}\,\uptriangle & \textbf{38.20}\,\uptriangle & \textbf{48.95}\,\uptriangle & \textbf{50.85}\,\uptriangle \\
 & & ToMCLIP($L_{\text{ta}}$) & 55.40\,\uptriangle & 54.95\,\uptriangle & \textbf{49.15}\,\uptriangle & 49.15\,\uptriangle & \textbf{57.35}\,\uptriangle & 53.50\,\uptriangle & \textbf{38.20}\,\uptriangle & 48.65\,\uptriangle & 50.79\,\uptriangle \\
 \textbf{Full data} & & ToMCLIP & \textbf{55.60}\,\uptriangle & \textbf{55.15}\,\uptriangle & 48.40\,\down & \textbf{50.00}\,\uptriangle & 56.70\,\uptriangle & 53.70\,\uptriangle & 38.00\,\uptriangle & 48.55\,\uptriangle & 50.76\,\uptriangle \\
\cmidrule(lr){2-12}
\textbf{(2M smaples)} & \multirow{5}{*}{\textbf{TR}} & CLIP & 58.55 & 29.10 & 17.15 & 10.80 & 0.80 & 4.25 & 5.25 & 2.15 & 16.01 \\
 & & MCLIP & 58.60 & 58.90 & 48.95 & 51.45 & 61.15 & 55.05 & 39.55 & 53.35 & 53.38 \\
 & & ToMCLIP($L_{\text{dm}}$) & 59.20\,\uptriangle & 59.35\,\uptriangle & 49.25\,\uptriangle & 51.80\,\uptriangle & 61.05\,\down & \textbf{56.50}\,\uptriangle & \textbf{40.75}\,\uptriangle & \textbf{54.15}\,\uptriangle & 54.01\,\uptriangle \\
 & & ToMCLIP($L_{\text{ta}}$) & 58.50\,\down & \textbf{60.15}\,\uptriangle & \textbf{49.70}\,\uptriangle & 51.70\,\uptriangle & 60.90\,\down & 55.20\,\uptriangle & 40.70\,\uptriangle & 53.80\,\uptriangle & 53.83\,\uptriangle \\
 & & ToMCLIP & \textbf{59.55}\,\uptriangle & 59.25\,\uptriangle & 49.55\,\uptriangle & \textbf{53.70}\,\uptriangle & \textbf{61.55}\,\uptriangle & 54.85\,\down & 40.70\,\uptriangle & 53.40\,\uptriangle & \textbf{54.07}\,\uptriangle \\
\midrule
 & \multirow{5}{*}{\textbf{IR}} & CLIP & \textbf{54.90} & 22.05 & 11.00 & 4.15 & 0.35 & 1.90 & 1.95 & 0.35 & 12.08 \\
 &  & MCLIP & 37.05 & 35.72 & 30.08 & 36.00 & 38.30 & 30.17 & 27.87 & 32.88 & 33.51 \\
 &  & \textsc{ToMCLIP}(\(L_{dm}\)) & 37.85\,\uptriangle & \textbf{37.27}\,\uptriangle & 30.65\,\uptriangle & \textbf{37.40}\,\uptriangle & \textbf{39.98}\,\uptriangle & 31.05\,\uptriangle & 28.17\,\uptriangle & \textbf{33.53}\,\uptriangle & 34.49\,\uptriangle \\
 &  & \textsc{ToMCLIP}(\(L_{ta}\)) & 38.00\,\uptriangle & 36.65\,\uptriangle & \textbf{31.23}\,\uptriangle & 36.55\,\uptriangle & 39.60\,\uptriangle & \textbf{31.27}\,\uptriangle & \textbf{29.17}\,\uptriangle & 33.50\,\uptriangle & 34.50\,\uptriangle \\
 \textbf{Low resource} &  & \textsc{ToMCLIP} & 37.10\,\uptriangle & 37.23\,\uptriangle & 30.55\,\uptriangle & 36.37\,\uptriangle & 38.85\,\uptriangle & 30.15\,\down & 28.48\,\uptriangle & 33.52\,\uptriangle & 34.03\,\uptriangle \\
\cmidrule(lr){2-12}
 \textbf{(1\% subset)} & \multirow{5}{*}{\textbf{TR}} & CLIP & \textbf{58.55} & 29.10 & 17.15 & 10.80 & 0.80 & 4.25 & 5.25 & 2.15 & 16.01 \\
 &  & MCLIP & 42.15 & 42.83 & 35.17 & 41.85 & 44.38 & 36.57 & 33.10 & 39.07 & 39.39 \\
 &  & \textsc{ToMCLIP}(\(L_{dm}\)) & 42.55\,\uptriangle & 42.48\,\down & \textbf{35.93}\,\uptriangle & 42.33\,\uptriangle & 45.72\,\uptriangle & 36.32\,\down & 32.92\,\down & 39.47\,\uptriangle & 39.71\,\uptriangle \\
 &  & \textsc{ToMCLIP}(\(L_{ta}\)) & 43.77\,\uptriangle & \textbf{43.37}\,\uptriangle & 35.90\,\uptriangle & \textbf{43.13}\,\uptriangle & \textbf{46.03}\,\uptriangle & \textbf{36.70}\,\uptriangle & \textbf{33.27}\,\uptriangle & \textbf{40.17}\,\uptriangle & 40.29\,\uptriangle \\
 &  & \textsc{ToMCLIP} & 42.92\,\uptriangle & 43.07\,\uptriangle & 35.02\,\down & 41.98\,\uptriangle & 45.17\,\uptriangle & 36.20\,\down & 32.65\,\down & 39.05\,\down & 39.51\,\uptriangle \\
\bottomrule
\end{tabular}}
\end{table*}

Table~\ref{tab:xflickrco_top1} reports multilingual retrieval performance (R@1) on xFlickr\&CO under both the  full-resource and low-resource settings,
where the results for the low-resource setting are averaged over three independent runs. 
In the full-resource setting, ToMCLIP consistently outperforms MCLIP across most languages 
for both IR and TR directions, achieving higher average R@1 (50.76\% on IR and 54.07\% on TR). 
In the low-resource setting, ToMCLIP still provides higher average R@1 (34.03\% on IR and 39.51\% on TR).

\subsection{Training Time and Evaluation Time}
\label{app:training_time}

To assess computational efficiency, we compared the average training time per epoch between the two models, MCLIP and ToMCLIP. 
We trained with one NVIDIA A100 (80 GB) on a single-node server (2× AMD EPYC 7513).
The baseline MCLIP required approximately 285 minutes per epoch, whereas the proposed ToMCLIP, which incorporates the additional topology loss and distance matrix alignment, required 357 minutes per epoch.
Although ToMCLIP increases the training cost relative to MCLIP, the additional overhead remains manageable considering the substantial improvement in cross-lingual alignment performance. 
This is made possible by our persistence diagram approximation strategy, which employs MST-based computation and graph sparsification to avoid the exponential complexity of constructing full Rips complexes.

Importantly, evaluation time remains unchanged between MCLIP and ToMCLIP. 
Since our method only modifies the training objective and does not alter the model architecture, no additional computation is introduced during inference. 
Thus, both models share identical evaluation speed and memory requirements, ensuring that the performance gains of ToMCLIP come at no cost during deployment.

\section{ADDITIONAL RESULTS WITH VIT-B/16 PLUS CLIP IMAGE ENCODER}
\label{app:vit+}

We replace the CLIP image encoder with ViT-B/16+~\citep{cherti2023reproducible}, which is trained on the LAION-400M dataset~\citep{schuhmann2021laion}. 
The multilingual text encoder remains XLM-RoBERTa~\citep{conneau2019unsupervised}, as in our main experiments. 
Except for the image backbone, the entire training and evaluation setup is identical to the setup described earlier.

For data, we use the publicly released precomputed text embeddings from \href{https://huggingface.co/datasets/M-CLIP/ImageCaptions-7M-Embeddings}{\texttt{ImageCaptions-7M-Embeddings}}, 
which contains 7M caption embeddings compatible with the ViT-B/16+ (by contrast, the corresponding ViT-B/16 release provides about 2M embeddings). 
All ViT-B/16+ runs use the full 7M set; under the low-resource condition, we uniformly subsample 1\% of these (\(\sim\)70K samples).

\paragraph{CIFAR-100 Zero-Shot Classification.}
Replacing the image backbone with ViT-B/16+ preserves the main trend: topology-aware objectives improve multilingual zero-shot accuracy over MCLIP in both regimes (Table~\ref{tab:cifar100_multilingual_avg_vitplus} and~\ref{tab:cifar100_vitplus}). 
On the \textbf{Full} setting, ToMCLIP($L_{\text{ta}}$) attains the best averages (Top-1/5/10 = 66.18/86.35/90.89) improving over MCLIP (64.54/85.30/89.99) by \textbf{+1.64/+1.05/+0.90} points, respectively. 
On the \textbf{Low} setting, the combined ToMCLIP model yields the highest averages (53.31/74.88/82.01) surpassing MCLIP (50.24/73.50/81.17) by \textbf{+3.07/+1.38/+0.84}. Notably, $L_{\text{ta}}$ alone also improves alignment quality under Low (51.42/74.47/81.97). 
These results are consistent with the main paper: enforcing topological consistency via $L_{\text{ta}}$ strengthens cross-lingual alignment in the shared embedding space.

\begin{table}[!ht]
\centering
\caption{Average Top-$k$ accuracy (\%) of the zero-shot classification on CIFAR-100 across 13 languages.}
\label{tab:cifar100_multilingual_avg_vitplus}
\small
\resizebox{0.7\textwidth}{!}{%
\begin{tabular}{lcccccc}
\toprule
{} & \multicolumn{3}{c}{\textbf{Low resource}} & \multicolumn{3}{c}{\textbf{Full data}} \\
\cmidrule(lr){2-4} \cmidrule(lr){5-7}
 & Top-1 & Top-5 & Top-10 & Top-1 & Top-5 & Top-10 \\
\midrule
CLIP & 24.39 & 35.91 & 42.47 & 24.39 & 35.91 & 42.47 \\
MCLIP & 50.24 & 73.50 & 81.17 & 64.54 & 85.30 & 89.99 \\
ToMCLIP($L_{\text{dm}}$) & 52.33 & 74.68 & 81.84 & 65.92 & 85.88 & 90.44 \\
ToMCLIP($L_{\text{ta}}$) & 51.42 & 74.47 & 81.97 & \textbf{66.18} & \textbf{86.35} & \textbf{90.89} \\
ToMCLIP & \textbf{53.31} & \textbf{74.88} & \textbf{82.01} & 65.53 & 85.82 & 90.33 \\
\bottomrule
\end{tabular}}
\end{table}
\begin{table*}[!ht]
\centering
\caption{Top-$k$ accuracy (\%) of zero-shot classification on CIFAR-100 across 13 languages (Full vs. Low). ViT-B/16+ is used for CLIP image encoder.}
\label{tab:cifar100_vitplus}
\resizebox{\textwidth}{!}{%
\begin{tabular}{llcccccccccccccc}
\toprule
\multirow{3}{*}{Setting} & \multirow{3}{*}{Model} & \multicolumn{13}{c}{Languages (13)} & \multirow{3}{*}{\textit{Avg}} \\
\cmidrule(lr){3-15}
& & En & Fr & Es & De & It & Ru & Pl & Tr & Da & Ja & Zh & Ko & Vi \\
\midrule
\multicolumn{16}{c}{Top-1 accuracy (\%) } \\
\midrule
& CLIP & 72.81 & 52.12 & 45.49 & 46.15 & 40.49 & 4.43 & 11.57 & 12.37 & 19.96 & 3.68 & 2.41 & 1.19 & 4.46 & 24.39 \\
& MCLIP & 72.42 & 66.85 & 69.25 & 56.04 & 69.85 & 67.57 & 64.14 & 65.87 & 69.86 & 38.09 & 69.11 & 66.28 & 63.75 & 64.54 \\
\textbf{Full data} & ToMCLIP($L_{\text{dm}}$) & \textbf{73.24} & \textbf{67.76} & 68.90 & 64.60 & 69.30 & \textbf{68.16} & \textbf{66.39} & \textbf{69.54} & 70.28 & 38.00 & \textbf{69.48} & 66.98 & 64.34 & 65.92 \\
\textbf{(2M smaples)} & ToMCLIP($L_{\text{ta}}$) & 72.21 & 67.31 & \textbf{69.61} & \textbf{69.30} & 68.50 & 67.30 & 64.74 & 67.93 & \textbf{70.39} & \textbf{41.10} & 68.75 & \textbf{67.51} & \textbf{65.69} & \textbf{66.18} \\
& ToMCLIP & 72.92 & \textbf{67.76} & 69.31 & 67.04 & \textbf{70.76} & 67.25 & 64.05 & 68.80 & 69.72 & 37.91 & 68.01 & 63.64 & 64.77 & 65.53 \\
\midrule
& CLIP & \textbf{72.81} & 52.12 & 45.49 & 46.15 & 40.49 & 4.43 & 11.57 & 12.37 & 19.96 & 3.68 & 2.41 & 1.19 & 4.46 & 24.39 \\
& MCLIP & 58.56 & 52.47 & 52.93 & 53.94 & 47.82 & 53.18 & 48.25 & 45.04 & 50.76 & 39.18 & 53.17 & 49.73 & 48.13 & 50.24 \\
\textbf{Low resource} & ToMCLIP($L_{\text{dm}}$) & 62.66 & 54.16 & 54.11 & 54.45 & 49.45 & 55.88 & 49.43 & 47.59 & 52.62 & 41.04 & 56.23 & 52.41 & \textbf{50.32} & 52.33 \\
\textbf{(1\% subset)} & ToMCLIP($L_{\text{ta}}$) & 62.46 & 54.78 & 53.41 & 54.75 & 49.09 & 50.94 & 49.22 & 45.68 & 52.68 & 39.26 & 56.01 & 51.16 & 48.97 & 51.42 \\
& ToMCLIP & 63.58 & \textbf{55.88} & \textbf{54.59} & \textbf{57.61} & \textbf{49.96} & \textbf{56.66} & \textbf{50.31} & \textbf{49.26} & \textbf{54.08} & \textbf{41.40} & \textbf{56.41} & \textbf{53.22} & 50.13 & \textbf{53.31} \\
\midrule
\multicolumn{16}{c}{Top-5 accuracy (\%) } \\
\midrule
& CLIP & 92.84 & 72.59 & 62.72 & 64.43 & 56.85 & 11.12 & 19.36 & 21.52 & 28.42 & 10.89 & 7.99 & 7.19 & 10.85 & 35.91 \\
& MCLIP & 92.81 & 87.94 & 90.49 & 82.12 & \textbf{89.68} & 88.43 & 83.43 & 88.71 & 87.86 & 50.84 & \textbf{90.96} & 88.57 & 87.04 & 85.30 \\
\textbf{Full data} & ToMCLIP($L_{\text{dm}}$) & 93.20 & \textbf{88.48} & \textbf{90.50} & 84.67 & 89.26 & 87.72 & \textbf{85.38} & 89.81 & \textbf{88.01} & 52.14 & 90.24 & 88.77 & \textbf{88.26} & 85.88 \\
\textbf{(2M smaples)} & ToMCLIP($L_{\text{ta}}$) & 93.04 & 88.01 & 89.32 & \textbf{89.76} & 89.30 & 87.78 & 83.99 & 89.84 & 87.99 & \textbf{55.97} & 90.37 & \textbf{89.02} & 88.21 & \textbf{86.35} \\
& ToMCLIP & \textbf{93.65} & 88.44 & 90.45 & 87.66 & 89.64 & \textbf{88.70} & 83.47 & \textbf{89.95} & 87.96 & 49.97 & \textbf{90.96} & 86.73 & 88.12 & 85.82 \\
\midrule
& CLIP & \textbf{92.84} & 72.59 & 62.72 & 64.43 & 56.85 & 11.12 & 19.36 & 21.52 & 28.42 & 10.89 & 7.99 & 7.19 & 10.85 & 35.91 \\
& MCLIP & 83.44 & 73.41 & 73.25 & 78.28 & 66.15 & 77.10 & 70.11 & 67.66 & 72.80 & 60.21 & 81.99 & 75.26 & 75.79 & 73.50 \\
\textbf{Low resource} & ToMCLIP($L_{\text{dm}}$) & 85.33 & 74.40 & 73.40 & 77.95 & 66.20 & \textbf{78.64} & \textbf{71.62} & 72.08 & 74.47 & 59.85 & 82.26 & \textbf{77.26} & 77.38 & 74.68 \\
\textbf{(1\% subset)} & ToMCLIP($L_{\text{ta}}$) & 84.97 & \textbf{74.63} & 74.56 & 79.33 & 66.21 & 76.99 & 71.23 & 69.11 & \textbf{75.00} & 60.28 & 82.80 & 75.45 & \textbf{77.52} & 74.47 \\
& ToMCLIP & 85.11 & 73.72 & \textbf{74.71} & \textbf{80.35} & \textbf{66.46} & 77.66 & 70.76 & \textbf{72.17} & 74.98 & \textbf{60.57} & \textbf{83.40} & 76.62 & 76.97 & \textbf{74.88} \\
\midrule
\multicolumn{16}{c}{Top-10 accuracy (\%) } \\
\midrule
& CLIP & 96.32 & 79.39 & 71.42 & 72.38 & 64.10 & 18.04 & 25.92 & 27.40 & 34.99 & 18.04 & 13.90 & 13.53 & 16.71 & 42.47 \\
& MCLIP & 96.41 & 92.03 & 94.25 & 89.52 & 93.35 & 92.51 & 88.74 & 93.26 & 92.10 & 56.07 & \textbf{95.35} & \textbf{94.28} & 91.96 & 89.99 \\
\textbf{Full data}  & ToMCLIP($L_{\text{dm}}$) & 96.65 & 92.69 & \textbf{94.68} & 89.92 & 93.35 & 92.72 & \textbf{90.17} & 93.96 & 91.31 & 58.17 & 94.65 & 94.26 & 93.15 & 90.44 \\
\textbf{(2M smaples)} & ToMCLIP($L_{\text{ta}}$) & 96.53 & 92.01 & 93.76 & \textbf{93.61} & \textbf{93.58} & 92.36 & 88.83 & \textbf{94.44} & \textbf{92.23} & \textbf{62.28} & 94.90 & 94.17 & 92.89 & \textbf{90.89} \\
& ToMCLIP & \textbf{96.72} & \textbf{92.90} & 94.14 & 92.45 & 93.40 & \textbf{93.70} & 88.01 & 94.14 & 91.54 & 55.21 & 95.15 & 93.43 & \textbf{93.47} & 90.33 \\

\midrule
& CLIP & \textbf{96.32} & 79.39 & 71.42 & 72.38 & 64.10 & 18.04 & 25.92 & 27.40 & 34.99 & 18.04 & 13.90 & 13.53 & 16.71 & 42.47 \\
& MCLIP & 91.16 & 79.26 & 81.08 & 85.98 & 72.05 & 84.87 & 79.89 & 77.03 & 80.56 & 67.11 & 89.01 & 82.35 & 84.89 & 81.17 \\
\textbf{Low resource} & ToMCLIP($L_{\text{dm}}$) & 91.16 & 80.00 & 80.47 & 86.07 & 72.30 & \textbf{86.25} & \textbf{80.06} & \textbf{81.06} & 82.00 & 66.52 & 88.75 & \textbf{83.52} & \textbf{85.75} & 81.84 \\
\textbf{(1\% subset)} & ToMCLIP($L_{\text{ta}}$) & 91.65 & \textbf{80.80} & \textbf{83.27} & 86.57 & \textbf{73.19} & 85.57 & 78.75 & 77.75 & 81.99 & \textbf{68.89} & 89.70 & 82.31 & 85.13 & 81.97 \\
& ToMCLIP & 91.54 & 79.59 & 81.88 & \textbf{87.54} & 72.80 & 85.15 & 79.71 & 80.38 & \textbf{82.13} & 67.88 & \textbf{89.79} & 83.41 & 84.37 & \textbf{82.01} \\
\bottomrule
\end{tabular}
}
\end{table*}

\paragraph{Multilingual Image–Text Retrieval on xFlickr\&CO.}
With the ViT-B/16+ image encoder, topology-aware objectives improve multilingual retrieval over MCLIP in most settings (Table~\ref{tab:xflickrco_vitplus}). 
On \textbf{Full}, ToMCLIP($L_{\text{ta}}$) attains the best averages for both directions 
(\textbf{IR}: R@1/5/10 = \textbf{62.98}/\textbf{85.79}/\textbf{91.60} vs. MCLIP: 62.24/85.27/91.09 and
\textbf{TR}: \textbf{63.79}/\textbf{86.21}/\textbf{91.98} vs. 62.82/85.47/91.32). 
On \textbf{Low}, the combined ToMCLIP variant yields the top averages for IR (R@1/5/10 = \textbf{58.53}/\textbf{83.37}/\textbf{90.51}), while ToMCLIP($L_{\text{dm}}$) is strongest for TR (R@1/5/10 = \textbf{57.99}/\textbf{83.84}/\textbf{90.63}). 
These trends mirror our zero-shot CIFAR-100 results: enforcing topological consistency via $L_{\text{ta}}$ improves cross-lingual alignment.

\begin{table}[!ht]
\centering
\caption{Multilingual retrieval on xFlickr\&CO. Average R@k (\%) across 8 languages (Low vs. Full).
\textcolor{red}{$\blacktriangle$} indicates an improvement over MCLIP (same setting and direction),
\textcolor{blue}{$\blacktriangledown$} indicates a decrease.}
\label{tab:xflickrco_vitplus}
\resizebox{0.9\textwidth}{!}{%
\begin{tabular}{llcccccc}
\toprule
\multirow{2}{*}{Direction} & \multirow{2}{*}{Model} & \multicolumn{3}{c}{\textbf{Low resource (1\% subset)}} & \multicolumn{3}{c}{\textbf{Full data (2M samples)}} \\
\cmidrule(lr){3-5} \cmidrule(lr){6-8}
& & R@1 & R@5 & R@10 & R@1 & R@5 & R@10 \\
\midrule
IR & CLIP & 16.38 & 27.00 & 32.06 & 16.38 & 27.00 & 32.06 \\
 & MCLIP & 56.44 & 82.28 & 89.60 & 62.24 & 85.27 & 91.09 \\
 & \textsc{ToMCLIP}(\(L_{dm}\)) & 57.91\,\deltaup{1.47} & 83.15\,\deltaup{0.87} & 90.37\,\deltaup{0.77} & 62.24\,\deltaup{0.00} & 85.39\,\deltaup{0.12} & 91.22\,\deltaup{0.13} \\
 & \textsc{ToMCLIP}(\(L_{ta}\)) & 57.58\,\deltaup{1.14} & 82.77\,\deltaup{0.49} & 90.12\,\deltaup{0.53} & \textbf{62.98}\,\deltaup{0.74} & \textbf{85.79}\,\deltaup{0.52} & \textbf{91.60}\,\deltaup{0.51}\\
 & \textsc{ToMCLIP} & \textbf{58.53}\,\deltaup{2.08} & \textbf{83.37}\,\deltaup{1.09} & \textbf{90.51}\,\deltaup{0.91} & 61.91\,\deltadown{0.33} & 84.89\,\deltadown{0.38} & 90.78\,\deltadown{0.31}\\
\midrule
TR & CLIP & 18.91 & 31.46 & 36.59 & 18.91 & 31.46 & 36.59 \\
 & MCLIP & 56.73 & 83.33 & 90.34 & 62.82 & 85.47 & 91.32 \\
 & \textsc{ToMCLIP}(\(L_{dm}\)) & \textbf{57.99}\,\deltaup{1.26} & \textbf{83.84}\,\deltaup{0.51} & \textbf{90.63}\,\deltaup{0.29} & 62.95\,\deltaup{0.13} & 85.67\,\deltaup{0.20} & 91.14\,\deltadown{0.17} \\
 & \textsc{ToMCLIP}(\(L_{ta}\)) & 57.33\,\deltaup{0.60} & 83.26\,\deltadown{0.06} & 90.27\,\deltadown{0.07} & \textbf{63.79}\,\deltaup{0.97} & \textbf{86.21}\,\deltaup{0.74} & \textbf{91.98}\,\deltaup{0.66} \\
 & \textsc{ToMCLIP} & 57.57\,\deltaup{0.84} & 83.39\,\deltaup{0.06} & 90.61\,\deltaup{0.28} & 62.19\,\deltadown{0.63} & 85.09\,\deltadown{0.38} & 90.84\,\deltadown{0.47} \\
\bottomrule
\end{tabular}}
\end{table}

\begin{table}[t]
\centering
\caption{Multilingual retrieval on xFlickr\&CO. R@1 retrieval accuracy (\%) across languages.
\textcolor{red}{$\blacktriangle$} and \textcolor{blue}{$\blacktriangledown$} mark improvements/decreases over MCLIP
for the same setting and direction; here only the icons are shown.}
\resizebox{0.9\textwidth}{!}{%
\begin{tabular}{lllcccccccc c}
\toprule
\multirow{3}{*}{Setting} & \multirow{3}{*}{Direction} &  \multirow{3}{*}{Model} & \multicolumn{8}{c}{Languages} & \multirow{2}{*}{\textit{Avg}} \\
\cmidrule(lr){4-11}
 & &  & En & Es & De & Id & Ru & Tr & Ja & Zh &  \\
\midrule
 & \multirow{5}{*}{\textbf{IR}} & CLIP & 64.70 & 34.70 & 21.35 & 5.65 & 0.90 & 2.70 & 0.40 & 0.65 & 16.38 \\
 &  & MCLIP & 65.50 & \textbf{69.05} & 59.60 & 61.40 & \textbf{72.45} & \textbf{66.90} & 41.40 & 61.60 & 62.24 \\
 &  & \textsc{ToMCLIP}(\(L_{dm}\)) & 65.60\,\uptriangle & 68.70\,\down & 59.65\,\uptriangle & 61.70\,\uptriangle & 72.30\,\down & 65.70\,\down & 42.35\,\uptriangle & 61.90\,\uptriangle & 62.24\,\uptriangle \\
 &  & \textsc{ToMCLIP}(\(L_{ta}\)) & 65.20\,\down & 69.00\,\down & \textbf{60.00}\,\uptriangle & \textbf{63.05}\,\uptriangle & 72.35\,\down & 65.75\,\down & \textbf{46.25}\,\uptriangle & 62.25\,\uptriangle & \textbf{62.98}\,\uptriangle \\
 \textbf{Full data} &  & \textsc{ToMCLIP} & \textbf{65.75}\,\uptriangle & 68.60\,\down & 59.20\,\down & 61.30\,\down & 72.20\,\down & 66.00\,\down & 39.90\,\down & \textbf{62.35}\,\uptriangle & 61.91\,\down \\
\cmidrule(lr){2-12}
 \textbf{(2M samples)} & \multirow{5}{*}{\textbf{TR}} & CLIP & 66.70 & 40.45 & 26.05 & 10.05 & 1.15 & 5.10 & 0.85 & 0.90 & 18.91 \\
 &  & MCLIP & 68.30 & 68.90 & 59.20 & 62.00 & 73.55 & 66.75 & 42.50 & 61.35 & 62.82 \\
 &  & \textsc{ToMCLIP}(\(L_{dm}\)) & 68.30\,\uptriangle & \textbf{70.00}\,\uptriangle & 57.70\,\down & 62.20\,\uptriangle & \textbf{73.75}\,\uptriangle & 66.85\,\uptriangle & 43.50\,\uptriangle & 61.30\,\down & 62.95\,\uptriangle \\
 &  & \textsc{ToMCLIP}(\(L_{ta}\)) & 68.80\,\uptriangle & 69.85\,\uptriangle & \textbf{59.75}\,\uptriangle & \textbf{62.40}\,\uptriangle & \textbf{73.75}\,\uptriangle & 66.60\,\down & \textbf{46.70}\,\uptriangle & \textbf{62.45}\,\uptriangle & \textbf{63.79}\,\uptriangle \\
 &  & \textsc{ToMCLIP} & \textbf{68.85}\,\uptriangle & 69.30\,\uptriangle & 57.85\,\down & 61.30\,\down & 72.65\,\down & \textbf{67.30}\,\uptriangle & 39.55\,\down & 60.75\,\down & 62.19\,\down \\
 \midrule
 & \multirow{5}{*}{\textbf{IR}} & CLIP & \textbf{64.70} & 34.70 & 21.35 & 5.65 & 0.90 & 2.70 & 0.40 & 0.65 & 16.38 \\
 &  & MCLIP & 59.05 & 60.30 & 52.45 & 55.85 & 63.45 & 55.15 & 49.30 & 56.00 & 56.44 \\
 &  & \textsc{ToMCLIP}(\(L_{dm}\)) & 59.50\,\uptriangle & \textbf{63.05}\,\uptriangle & \textbf{55.30}\,\uptriangle & \textbf{57.05}\,\uptriangle & 64.80\,\uptriangle & 55.80\,\uptriangle & \textbf{49.75}\,\uptriangle & 58.05\,\uptriangle & 57.91\,\uptriangle \\
 &  & \textsc{ToMCLIP}(\(L_{ta}\)) & 59.50\,\uptriangle & 61.95\,\uptriangle & 53.75\,\uptriangle & 56.80\,\uptriangle & 65.80\,\uptriangle & 55.80\,\uptriangle & 49.65\,\uptriangle & 57.40\,\uptriangle & 57.58\,\uptriangle \\
 \textbf{Low resource} &  & \textsc{ToMCLIP} & 60.55\,\uptriangle & 62.80\,\uptriangle & 55.25\,\uptriangle & 57.00\,\uptriangle & \textbf{66.60}\,\uptriangle & \textbf{57.40}\,\uptriangle & 49.55\,\uptriangle & \textbf{59.05}\,\uptriangle & \textbf{58.53}\,\uptriangle \\
\cmidrule(lr){2-12}
\textbf{(1\% subset)} & \multirow{5}{*}{\textbf{TR}} & CLIP & \textbf{66.70} & 40.45 & 26.05 & 10.05 & 1.15 & 5.10 & 0.85 & 0.90 & 18.91 \\
 &  & MCLIP & 60.35 & 61.05 & 51.85 & 56.40 & 63.55 & 54.70 & 49.05 & 56.90 & 56.73 \\
 &  & \textsc{ToMCLIP}(\(L_{dm}\)) & 61.45\,\uptriangle & 61.70\,\uptriangle & \textbf{53.45}\,\uptriangle & \textbf{57.20}\,\uptriangle & \textbf{65.45}\,\uptriangle & \textbf{55.85}\,\uptriangle & \textbf{51.30}\,\uptriangle & \textbf{57.55}\,\uptriangle & \textbf{57.99}\,\uptriangle \\
 &  & \textsc{ToMCLIP}(\(L_{ta}\)) & 61.45\,\uptriangle & 61.35\,\uptriangle & 53.10\,\uptriangle & 56.25\,\down & 64.40\,\uptriangle & 55.70\,\uptriangle & 50.00\,\uptriangle & 56.40\,\down & 57.33\,\uptriangle \\
 &  & \textsc{ToMCLIP} & 61.15\,\uptriangle & \textbf{61.80}\,\uptriangle & 52.80\,\uptriangle & 57.10\,\uptriangle & 65.30\,\uptriangle & 55.45\,\uptriangle & 49.70\,\uptriangle & 57.30\,\uptriangle & 57.57\,\uptriangle \\
\bottomrule
\end{tabular}}
\label{tab:xflickrco_per_language}
\end{table}

\section{ADDITIONAL RESULTS ON IMAGENET-1K}
We have conducted additional experiments on ImageNet-1K zero-shot classification. 
The results consistently show performance improvements over MCLIP in both full-data (2M samples) and low-resource (1\% subset) settings, further supporting the effectiveness and generality of the proposed topological alignment approach (Table~\ref{tab:imagenet_multilingual_avg} \~~\ref{tab:imagenet_lang_top10}).
We further note that the relative improvements are larger in the low-resource setting. 
For example, ToMCLIP (ViT-B/16+) achieves a +2.51\% Top-1 gain over MCLIP in English under the 1\% data condition, compared to +1.79\% in the full-data setting. This trend is consistent across multiple languages. 
This pattern suggests that the topological alignment loss acts as a structural regularizer, encouraging global embedding consistency across languages and providing greater benefit when training data is scarce.

\begin{table}[!ht]
\centering
\caption{Average Top-$k$ accuracy (\%) of the zero-shot classification on ImageNet-1K across 13 languages (ViT-B/32).}
\label{tab:imagenet_multilingual_avg}
\small
\resizebox{0.7\textwidth}{!}{%
\begin{tabular}{lcccccc}
\toprule
{} & \multicolumn{3}{c}{\textbf{Low resource}} & \multicolumn{3}{c}{\textbf{Full data}} \\
\cmidrule(lr){2-4} \cmidrule(lr){5-7}
 & Top-1 & Top-5 & Top-10 & Top-1 & Top-5 & Top-10 \\
\midrule
MCLIP & 11.98 & 27.33 & 35.05 & 34.01 & 56.56 & 64.16 \\
ToMCLIP ($L_{dm}$) & 12.34 & 27.77 & 35.31 & 34.40 & 56.99 & 64.44 \\
ToMCLIP ($L_{ta}$) & 12.51 & 27.98 & 35.55 & 34.07 & 56.69 & 64.22 \\
ToMCLIP & 12.18 & 27.57 & 35.22 & 34.37 & 56.87 & 64.45 \\
\bottomrule
\end{tabular}}
\end{table}
\begin{table}[!ht]
\centering
\caption{Average Top-$k$ accuracy (\%) of the zero-shot classification on ImageNet-1K across 13 languages (ViT-B/32+).}
\label{tab:imagenet_multilingual_avg_plus}
\small
\resizebox{0.7\textwidth}{!}{%
\begin{tabular}{lcccccc}
\toprule
{} & \multicolumn{3}{c}{\textbf{Low resource}} & \multicolumn{3}{c}{\textbf{Full data}} \\
\cmidrule(lr){2-4} \cmidrule(lr){5-7}
 & Top-1 & Top-5 & Top-10 & Top-1 & Top-5 & Top-10 \\
\midrule
MCLIP & 19.65 & 38.25 & 46.11 & 39.63 & 61.39 & 67.86 \\
ToMCLIP ($L_{dm}$) & 20.70 & 39.46 & 47.09 & 40.55 & 62.05 & 68.54 \\
ToMCLIP ($L_{ta}$) & 20.18 & 39.78 & 46.67 & 39.83 & 61.48 & 68.14 \\
ToMCLIP & 21.44 & 40.15 & 47.89 & 40.43 & 61.85 & 68.34 \\
\bottomrule
\end{tabular}}
\end{table}
\begin{table*}[!ht]
\centering
\caption{Top-1 accuracy (\%) of zero-shot classification on ImageNet-1K across 13 languages (ViT-B/32+).}
\label{tab:imagenet_lang_top1}
\resizebox{\textwidth}{!}{%
\begin{tabular}{llcccccccccccccc}
\toprule
& & \multicolumn{13}{c}{Languages (13)} & \multirow{2}{*}{\textit{Avg}} \\
\cmidrule(lr){3-15}
Setting & Model & En & Fr & Es & De & It & Ru & Pl & Tr & Da & Ja & Zh & Ko & Vi \\
\midrule
& MCLIP & 48.36 & 41.74 & 44.74 & 42.30 & 39.03 & 38.78 & 38.30 & 42.95 & 43.75 & 20.20 & 36.29 & 42.60 & 36.19 & 39.63 \\
\textbf{Full data} & ToMCLIP ($L_{dm}$) & 47.80 & 42.01 & 45.02 & 42.81 & 41.04 & 39.83 & 39.08 & 43.41 & 44.57 & 23.09 & 37.40 & 43.51 & 37.55 & 40.55 \\
\textbf{(2M smaples)} & ToMCLIP ($L_{ta}$) & 46.70 & 41.14 & 43.79 & 41.87 & 39.73 & 38.97 & 38.54 & 41.97 & 42.71 & 26.68 & 37.71 & 42.17 & 35.80 & 39.83 \\
& ToMCLIP & 48.60 & 42.55 & 44.74 & 42.77 & 40.61 & 40.14 & 38.69 & 43.71 & 44.83 & 21.93 & 36.33 & 43.44 & 37.19 & 40.43 \\
\midrule
& MCLIP & 24.45 & 20.49 & 20.15 & 20.89 & 18.39 & 19.91 & 19.41 & 17.30 & 21.15 & 17.10 & 21.82 & 16.24 & 18.08 & 19.65 \\
\textbf{Low resource} & ToMCLIP ($L_{dm}$) & 25.67 & 21.67 & 21.71 & 22.32 & 20.12 & 22.12 & 20.00 & 18.77 & 22.17 & 17.78 & 22.12 & 15.75 & 18.90 & 20.70 \\
\textbf{(1\% subset)} & ToMCLIP ($L_{ta}$) & 24.77 & 20.72 & 20.68 & 20.73 & 19.78 & 21.19 & 19.24 & 18.61 & 20.94 & 18.08 & 22.63 & 15.92 & 19.01 & 20.18 \\
& ToMCLIP & 26.96 & 22.22 & 22.97 & 23.25 & 20.71 & 22.44 & 20.40 & 19.48 & 22.58 & 18.69 & 23.59 & 15.44 & 20.00 & 21.44 \\
\bottomrule
\end{tabular}
}
\end{table*}
\begin{table*}[!ht]
\centering
\caption{Top-10 accuracy (\%) of zero-shot classification on ImageNet-1K across 13 languages (ViT-B/32+).}
\label{tab:imagenet_lang_top10}
\resizebox{\textwidth}{!}{%
\begin{tabular}{llcccccccccccccc}
\toprule
& & \multicolumn{13}{c}{Languages (13)} & \multirow{2}{*}{\textit{Avg}} \\
\cmidrule(lr){3-15}
Setting & Model & En & Fr & Es & De & It & Ru & Pl & Tr & Da & Ja & Zh & Ko & Vi \\
\midrule
& MCLIP & 78.09 & 70.31 & 72.65 & 72.05 & 65.82 & 64.92 & 64.70 & 72.86 & 72.57 & 43.25 & 67.58 & 70.54 & 66.80 & 67.86 \\
\textbf{Full data} & ToMCLIP ($L_{dm}$) & 77.94 & 69.86 & 72.85 & 71.64 & 66.79 & 66.30 & 64.78 & 73.12 & 73.27 & 47.11 & 68.49 & 70.91 & 68.01 & 68.54 \\
\textbf{(2M smaples)} & ToMCLIP ($L_{ta}$) & 76.51 & 68.92 & 71.94 & 70.78 & 65.58 & 65.62 & 64.71 & 70.14 & 70.64 & 53.26 & 69.79 & 70.33 & 67.56 & 68.14 \\
& ToMCLIP & 78.54 & 70.41 & 73.05 & 72.07 & 66.07 & 66.41 & 64.07 & 73.19 & 72.62 & 45.81 & 68.06 & 71.11 & 66.97 & 68.34 \\
\midrule
& MCLIP & 54.32 & 46.54 & 46.25 & 47.88 & 42.40 & 45.48 & 43.66 & 42.35 & 49.09 & 40.34 & 53.53 & 39.66 & 47.86 & 46.11 \\
\textbf{Low resource} & ToMCLIP ($L_{dm}$) & 55.89 & 47.80 & 48.63 & 48.97 & 43.81 & 47.28 & 43.07 & 43.31 & 48.96 & 41.38 & 54.30 & 39.56 & 49.26 & 47.09 \\
\textbf{(1\% subset)} & ToMCLIP ($L_{ta}$) & 55.16 & 47.32 & 47.78 & 45.83 & 44.40 & 46.11 & 42.45 & 43.89 & 47.69 & 42.02 & 56.07 & 38.14 & 49.84 & 46.67 \\
& ToMCLIP & 57.57 & 48.79 & 48.88 & 49.67 & 44.83 & 47.61 & 44.25 & 44.81 & 48.91 & 43.05 & 56.06 & 36.64 & 51.43 & 47.89 \\
\bottomrule
\end{tabular}
}
\end{table*}

\end{document}